\documentclass[a4paper, 11pt, twoside]{article}
\usepackage{fullpage}

\usepackage[utf8]{inputenc} 
\usepackage[T1]{fontenc}    

\input{preamble_marina-n.tex}

\newcommand{\squeeze}{}
\usepackage[authoryear]{natbib}

\begin{document}

\title{\textbf{MARINA-P: Superior Performance in Non-smooth Federated Optimization with Adaptive Stepsizes}}

\author{Igor Sokolov$^{1}\thanks{Corresponding author: {igor.sokolov.1@kaust.edu.sa}.}$ \quad \quad Peter Richt\'{a}rik$^1$}
\date{$^1$ Center of Excellence for Generative AI,\\ King Abdullah University of Science and Technology,\\ Saudi Arabia}
    
\maketitle

\begin{abstract}
Non-smooth communication-efficient federated optimization is crucial for many practical machine learning applications, yet it remains largely unexplored theoretically. Recent advancements in communication-efficient methods have primarily focused on smooth convex and non-convex regimes, leaving a significant gap in our understanding of the more challenging non-smooth convex setting. Additionally, existing federated optimization literature often overlooks the importance of efficient server-to-worker communication (downlink), focusing primarily on worker-to-server communication (uplink).
In this paper, we consider a setup where uplink communication costs are negligible and focus on optimizing downlink communication by improving the efficiency of recent state-of-the-art downlink schemes such as \algname{EF21-P} \citep{EF21-P} and \algname{MARINA-P} \citep{gruntkowska2024improving} in the non-smooth convex setting. We address these gaps through several key contributions. First, we extend the non-smooth convex theory of \algname{EF21-P} \citep{NonsmoothEF21-P}, originally developed for single-node scenarios, to the distributed setting. Second, we extend existing results for \algname{MARINA-P} to the non-smooth convex setting.
For both algorithms, we prove an optimal $\cO\rb{\nfr{1}{\sqrt{T}}}$ convergence rate under standard assumptions and establish communication complexity bounds that match those of classical subgradient methods. Furthermore, we provide theoretical guarantees for both \algname{EF21-P} and \algname{MARINA-P} under constant, decreasing, and adaptive (Polyak-type) stepsizes. Our experiments demonstrate that \algname{MARINA-P}, when used with correlated compressors, outperforms other methods not only in smooth non-convex settings (as originally shown by \citet{gruntkowska2024improving}) but also in non-smooth convex regimes.
To the best of our knowledge, this work presents the first theoretical results for distributed non-smooth optimization incorporating server-to-worker compression, along with comprehensive analysis for various stepsize schemes.
\end{abstract}

\tableofcontents

\section{Introduction}
In recent years, the machine learning community has witnessed a paradigm shift towards larger models and more extensive datasets, driving significant advancements in model performance and practical utility \citep{lecun2015deep}. However, this ``big model'' and ``big data'' approach introduce new challenges in hardware requirements, algorithmic design, and software infrastructure necessary for efficient model training \citep{BotCurNoce, kaplan2020scaling, deng2009imagenet}.

\paragraph{The Rise of Big Data and Distributed Systems.}
The sheer volume of data required for state-of-the-art models has necessitated the adoption of distributed computing systems \citep{Dean2012, DCGD, lin2018deep}. Traditional single-machine approaches are no longer feasible due to storage and computational limitations, leading to the distribution of datasets across multiple parallel workers. This distributed paradigm is particularly evident in supervised learning problems \citep{ hastie2009elements, shalev2014understanding, vapnik2013nature}, which can be formulated as:
\begin{eqnarray}\label{eq:finite_sum}
\min\limits_{x \in \mathbb{R}^d}\left\{f(x)\eqdef\frac{1}{n} \sum\limits_{i=1}^n f_i(x)\right\},    
\end{eqnarray}
where $n$ is the number of clients (or devices), $x$ denotes the $d$-dimensional parameter vector of the machine learning model, and $f_i(x)$ represents the loss of model $x$ on the training data stored on client $i$. In this work, we assume that the functions $f_i$ are convex (possibly non-smooth) for all $i \in[n]:=\{1, \ldots, n\}$.

The distributed nature of data collection and processing has given rise to federated learning (FL) \citep{FedAvg2016, FEDLEARN, FEDOPT, FL2017-AISTATS}, a paradigm where heterogeneous clients collaboratively train a model using diverse, decentralized data while preserving privacy. This approach aims to eliminate the need for centralized data aggregation, addressing privacy concerns in sensitive applications.
In the federated learning setup, devices link directly to a central server that manages the optimization procedure \citep{FEDLEARN, kairouz2021advances}. The devices conduct local calculations on their individual datasets and transmit the outcomes (such as model modifications) to the server. The server then aggregates the incoming data, executes global calculations, and returns the updated model parameters to the devices. This cycle continues until the model reaches convergence or achieves a satisfactory level of performance.

\paragraph{Communication Challenges in Large-scale Model Training.}
While distributing data across workers addresses storage and compute limitations, it introduces significant communication overhead. Modern gradient-based training methods \citep{bottou2012stochastic, ADAM, demidovich2023guide, duchi2011adaptive, robbins1951stochastic} require iterative updates to all $d$ model parameters, making the communication of high-dimensional gradients prohibitively expensive as models scale. Researchers have proposed two main approaches to mitigate this challenge: performing multiple local gradient steps before communication, exemplified by algorithms such as \algname{LocalSGD} \citep{localSGD-Stich, localSGD-AISTATS2020, Blake2020, yi2024cohort, sadiev2022federated, richtarik2024unified}, and applying lossy transformations to gradient information to reduce data transfer volumes \citep{DCGD, alistarh2018convergence, mishchenko202099, DIANA, ADIANA, CANITA, richtarik2021ef21, fatkhullin2021ef21, richtarik20223pc, Seide2014, alistarh2017qsgd, panferov2024correlated}.

While worker-to-server communication has been extensively studied, server-to-worker communication often receives less attention. However, studies of 4G LTE and 5G networks \citep{huang2012close, narayanan2021variegated} demonstrate that upload and download speeds in mobile environments are often comparable, with differences rarely exceeding an order of magnitude. This observation necessitates optimization strategies that address communication efficiency in both directions.
\paragraph{Prevalence of Non-smooth Objectives in Machine Learning Applications.}
While distributed optimization has made significant progress, theoretical analysis has primarily focused on smooth objectives, leaving non-smooth optimization understudied in federated settings. Non-smoothness is inherent in many machine learning applications: ReLU activation functions in deep learning \citep{Glorot2011, Nair2010}, L1 regularization for parameter sparsity \citep{Tibshirani1996, Zou2005}, hinge loss in support vector machines \citep{Cortes1995}, and total variation regularization in computer vision \citep{Rudin1992, Chambolle2004}. Additional examples include quantile regression \citep{Koenker1978}, max-pooling layers in convolutional networks \citep{Scherer2010}, submodular function minimization \citep{Bach2013}, Huber loss in robust optimization \citep{Huber1964}, and graph-based learning algorithms \citep{Hallac2015}.

\paragraph{Adaptive Stepsizes are Widely Used in Practice.}
Since theoretically assumed constants, such as those for L-Lipschitz continuity or smoothness, are often infeasible to determine for deep neural networks, adaptive learning rate methods have become ubiquitous in training. Popular algorithms include AdaGrad \citep{duchi2011adaptive}, RMSProp, Adam \citep{ADAM}, and AMSGrad \citep{reddi2018convergence}, which dynamically adjust learning rates for each parameter based on observed gradients. Modern deep learning frameworks like PyTorch \citep{paszke2019pytorch} offer a variety of learning rate schedulers, such as \texttt{StepLR}, \texttt{MultiStepLR}, \texttt{ExponentialLR}, and \texttt{CosineAnnealingLR} \citep{schmidt2021descending}, facilitating easy implementation of diverse learning rate strategies.

\paragraph{Addressing Key Challenges in Federated Learning Simultaneously.}
Practical implementation of federated learning necessitates addressing three critical aspects simultaneously: non-smooth of the loss function, adaptive stepsizes, and communication efficiency. 
By addressing these challenges concurrently, we aim to expand federated learning's applicability to a broader range of real-world problems with non-differentiable loss functions.

\subsection{Notation and Assumptions}
We denote the set $\{1,2, \cdots, n\}$ by $[n]$. For vectors, $\twonorm{\cdot}$ represents the Euclidean norm, while for matrices, it denotes the spectral norm. The inner product of vectors $u$ and $v$ is denoted by $\langle u, v\rangle$. We use $\cO(\cdot)$ to hide absolute constants. 
We denote $R_0 \eqdef \twonorm{x^0 - x^*}.$

Our analysis relies on the following standard assumptions:
\begin{assumption}\label{as:existence_of_minimizer}
The function $f$ has at least one minimizer, denoted by $x^*$.
\end{assumption}
\begin{assumption}\label{as:fi_convexity}
The functions $f_i$ are convex for all $i\in [n]$.
\end{assumption}
In the distributed setting, assuming convexity for individual functions $f_i$ is sufficient, as it implies convexity for $f$ itself.
\begin{assumption}[Lipschitz continuity of $f_i$]\label{as:fi_lipschitzness}
Functions $f_i$ are $L_{0,i}$-Lipschitz continuous for all $i\in [n]$. That is, for all $i\in [n]$, there exists $L_{0,i} > 0$ such that
$\abs{f_i(x) - f_i(y)} \leq L_{0,i} \twonorm{x - y}, \quad \forall x, y \in \mathbb{R}^d.$
\end{assumption}
If each $f_i$ is Lipschitz continuous, then by Jensen's inequality, $f$ is $L_0$-Lipschitz with $L_0 \eqdef \suminn L_{0,i}$ \citep{Nesterov2013}. 

Both convexity and Lipschitz continuity of $f$ are standard assumptions in non-smooth optimization \citep{vorontsova2021convex, Nesterov2013, bubeck2015convex,beck2017first, duchi2018introductory, lan2020first, drusvyatskiy2020convex}.
Moreover, $L_0$ and $L_{0,i}$-Lipschitz continuity imply uniformly bounded subgradients \citep{beck2017first}, a property that will be useful in our proofs:
\begin{eqnarray}
\twonorm{\partial f(x)} &\le& L_0 \quad \forall x \in \mathbb{R}^d, \\
\twonorm{\partial f_i(x)} &\le& L_{0,i} \quad \forall x \in \mathbb{R}^d \;\;\text{and}\;\; \forall i\in [n].
\end{eqnarray}
We define $\wt{L}_0 \eqdef \sqrt{\frac{1}{n} \sum_{i=1}^n L_{0,i}^2}$ and $\avg{L}_0 \eqdef \frac{1}{n} \sum_{i=1}^n L_{0,i}$. By the arithmetic-quadratic mean inequality, we have $\avg{L}_0 \leq \widetilde{L}_0$.

Following classical optimization literature \citep{Nemirovski2009, beck2017first, duchi2018introductory, lan2020first, drusvyatskiy2020convex}, for non-smooth convex objectives, we aim to find an $\eps$-suboptimal solution: a random vector $\hat{x}\in \R^d$ satisfying 
\begin{eqnarray}\label{eq:suboptimal_condition}
\Exp{f(\hat{x}) - f(x^*)} \leq \eps,
\end{eqnarray}
where $\Exp{\cdot}$ denotes the expectation over algorithmic randomness.

To assess the efficiency of distributed subgradient-based algorithms, we primarily use two metrics:

\textit{1. Communication complexity} (alternatively, communication cost): The expected total number of floats per worker required to communicate to reach an $\eps$-suboptimal solution. In this paper, we focus on server-to-worker communication compression.

\textit{2. Iteration complexity}: The number of communication rounds needed to achieve an $\eps$-suboptimal solution.
\subsection{Related work}
\paragraph{Subgradient Methods in Non-smooth Convex Optimization.}
Subgradient methods, originating in the 1960s, are fundamental for solving non-smooth convex optimization problems \citep{Shor1985,Polyak1987}. Classical convergence analysis establishes optimal \citep{Nesterov2013, vorontsova2021convex} rates of $\cO\rb{\nfr{1}{\sqrt{T}}}$ in the general convex case \citep{Nesterov2013, vorontsova2021convex, Boyd2003, bubeck2015convex, beck2017first, duchi2018introductory, lan2020first, drusvyatskiy2020convex}, and $\cO\rb{\nfr{1}{{T}}}$ for strongly convex functions \citep{beck2017first, lan2020first, drusvyatskiy2020convex}. For stochastic settings, \citet{Nemirovski2009} developed robust mirror descent stochastic approximation methods achieving non-asymptotic $\cO\rb{\nfr{1}{\sqrt{T}}}$ convergence rates.
These rates assume a known iteration count $T$ with constant stepsizes proportional to $\nfr{1}{\sqrt{T}}$ (convex case) and $\nfr{1}{T}$ (strongly convex case). For unknown $T$, decreasing stepsizes of order $\cO\rb{\nfr{1}{\sqrt{t}}}$ and $\cO\rb{\nfr{1}{{t}}}$ introduce an additional logarithmic factor, yielding rates of $\cO\rb{\nfr{\log T}{\sqrt{T}}}$ \citep{Nesterov2013} and $\cO\rb{\nfr{\log T}{{T}}}$ \citep{hazan2007logarithmic, hazan2014beyond}, respectively. Recent advances have eliminated these logarithmic factors: \citet{zhu2024convergence} achieved the optimal $\cO\rb{\nfr{1}{\sqrt{T}}}$ rate for convex functions, while \citet{LacosteJulien2012} and \citet{Rakhlin2011} established the optimal $\cO\rb{\nfr{1}{{T}}}$ rate for strongly convex functions. Beyond ergodic convergence, several works \citep{Jain2019, zamani2023exact} have provided tight analyses for last-iterate convergence.
In machine learning, subgradient methods have demonstrated practical relevance in large-scale problems such as support vector machines and structured prediction \citep{ShalevShwartz2007, Ratliff2007}.
\paragraph{Communication Compression.}
Before discussing more advanced optimization methods, let us consider the simplest baseline: the standard subgradient method (\algname{SM}) 
\footnote{In this paper, we use the non-normalized form \eqref{eq:sm_non-norm} of the subgradient method studied in \citep{vorontsova2021convex, bubeck2015convex, beck2017first, duchi2018introductory, lan2020first, drusvyatskiy2020convex, Nemirovski2009}. Earlier works \citep{Shor1985, Polyak1987} typically employed SM in the form $x^{t+1}=x^t-\frac{\gamma_t}{\norm{\partial f(x^t)}}\partial f(x^t)$, which includes an additional normalization term $\norm{\partial f(x^t)}$.}
, which iteratively performs updates
\footnote{For constrained optimization problems, the subgradient method typically operates through projections onto a convex set $\cX$ (see \citep{bubeck2015convex, LacosteJulien2012, beck2017first, duchi2018introductory}). However, when optimizing over an unbounded domain, i.e., $\cX = \R^d$, projections are not needed.}
\begin{eqnarray}\label{eq:sm_non-norm}
\squeeze
x^{t+1}=x^t-\frac{\gamma_t}{n} \sum_{i=1}^n g_i^t,
\end{eqnarray}
where $g_i^t = \partial f_i(x^t)$ is a subgradient of $f_i$ at $x^t$.
In the distributed setting, the method can be implemented as follows: each worker calculates $g_i^t$ and sends it to the server, where the subgradients are aggregated. The server takes the step and broadcasts $x^{t+1}$ back to the workers. With stepsize $\gamma_t \eqdef \nfr{R_0}{L_0\sqrt{T}}$, where $R_0\eqdef \twonorm{x^0-x^*}$ and $T$ is the total number of iterations, \algname{SM} finds an $\varepsilon$-approximate solution after
$\cO\left(\nfr{L_0^2R_0^2}{\varepsilon^2}\right)$
steps \citep{Nesterov2013, drusvyatskiy2020convex}. Since at each step the workers and the server send $\Theta(d)$ coordinates/floats, the worker-to-server and server-to-worker communication costs are
$\cO\left(\nfr{d L_0^2R_0^2}{\varepsilon^2}\right).$
To formally quantify communication costs, we introduce the following definition.
\begin{definition}
The worker-to-server (w2s, uplink) and server-to-worker (s2w, downlink) communication complexities of a method are the expected number of coordinates/floats that a worker sends to the server and that the server sends to a worker, respectively, to find an $\eps$–solution.
\end{definition}
Communication compression techniques, such as sparsification \citep{wang2018atomo, mishchenko202099, alistarh2018convergence, wangni2018gradient, konevcny2018randomized} and quantization \citep{alistarh2017qsgd, wen2017terngrad, zhang2016zipml, horvath2022natural, wu2018error, DIANA}, are known to be immensely powerful for reducing the communication overhead of gradient-type methods.
Existing literature primarily considers two main classes of compression operators: \textit{unbiased} and \textit{biased (contractive)} compressors.
\begin{definition}\label{def:unbiased_compressor}
(Unbiased compressor). A stochastic mapping $\mathcal{Q}: \mathbb{R}^d \rightarrow \mathbb{R}^d$ is called an unbiased compressor/compression operator if there exists $\omega \geq 0$ such that for any $x \in \mathbb{R}^d$:
\begin{eqnarray}\label{eq:unbiased_compressor}
\squeeze
\mathbb{E}[\mathcal{Q}(x)]=x, \quad \mathbb{E}\left[\twonorm{\mathcal{Q}(x)-x}^2\right] \leq \omega\twonorm{x}^2 .
\end{eqnarray}
\end{definition}
This definition encompasses a wide range of well-known compression techniques, including Rand$K$ sparsification \citep{stich2018sparsified}, random dithering \citep{roberts1962picture, goodall1951television}, and natural compression \citep{horvath2022natural}. Notable examples of methods employing compressor \eqref{eq:unbiased_compressor} are \algname{QSGD}\citep{alistarh2017qsgd}, \algname{DCGD}\citep{DCGD}, \algname{MARINA}\citep{MARINA}, \algname{DIANA}\citep{DIANA}, \algname{VR-DIANA}\citep{DIANA2}, \algname{DASHA}\citep{DASHA}, \algname{FedCOMGATE}\citep{FedCOMGATE}, \algname{FedPAQ}\citep{reisizadeh2020fedpaq}, \algname{FedSTEPH}\citep{das2020improved}, \algname{FedCOM}\citep{FedCOMGATE}, \algname{ADIANA}\citep{ADIANA}, \algname{NEOLITHIC}\citep{NEOLITHIC}, \algname{ACGD}\citep{ADIANA}, and \algname{CANITA}\citep{CANITA}.
However, Definition \ref{def:unbiased_compressor} does not cover another important class of practically more favorable compressors, called \textit{contractive}, which are usually biased.
\begin{definition}\label{def:contractive_compressor}
(Contractive compressor). A stochastic mapping $\cC: \mathbb{R}^d \rightarrow \mathbb{R}^d$ is called a contractive compressor/compression operator if there exists $\alpha \in(0,1]$ such that for any $x \in \mathbb{R}^d$:
\begin{eqnarray}\label{eq:contractive_compressor}
\squeeze
\mathbb{E}\left[\twonorm{\cC(x)-x}^2\right] \leq(1-\alpha)\twonorm{x}^2 .
\end{eqnarray}
\end{definition}

We denote the families of compressors satisfying Definitions \ref{def:unbiased_compressor} and \ref{def:contractive_compressor} by $\mathbb{U}(\omega)$ and $\mathbb{B}(\alpha)$, respectively. Notably, it can easily be verified (see Lemma 8 in \citep{richtarik2021ef21}) that if $\cQ \in \mathbb{U}(\omega)$, then $(\omega+1)^{-1} \cQ \in \mathbb{B}\left((\omega+1)^{-1}\right)$, indicating that the family of biased compressors is wider.

Inequality \eqref{eq:contractive_compressor} is satisfied by a vast array of compressors considered in the literature, including numerous variants of sparsification operators, such as Top$K$ compressors \citep{Strom2015, dryden2016communication, aji2017sparse, alistarh2018convergence}, quantization operators \citep{alistarh2017qsgd, horvath2022natural}, low-rank approximation \citep{vogels2019powersgd, vogels2020practical, safaryan2021fednl}, count-sketches \citep{ivkin2019communication, rothchild2020fetchsgd}, and more. For a comprehensive overview of biased and unbiased compressors, we refer readers to the summary by \citep{beznosikov2023biased, demidovich2023guide, safaryan2022uncertainty}.

However, naive implementation of distributed \algname{SGD} with biased compression (e.g., TopK) can lead to exponential divergence \citep{beznosikov2023biased}. Error Feedback (hereafter \algname{EF14}), first proposed as a heuristic by \citet{Seide2014}, emerged as a crucial technique to address these divergence issues. Initial theoretical analysis of \algname{EF14} focused on single-node settings \citep{stich2018sparsified, Alistarh-EF-NIPS2018,stich2019error} before extending to distributed data regimes \citep{cordonnier2018convex, beznosikov2023biased,koloskova2020unified}. \citet{richtarik2021ef21} re-engineered \algname{EF14} into a new method called \algname{EF21}, achieving optimal $\cO\rb{\nfr{1}{T}}$ convergence rates for smooth non-convex problems under standard assumptions, improving upon the previous best-known rate of $\cO\rb{\nfr{1}{T^{2/3}}}$ \citep{koloskova2020unified}.

The \algname{EF21} framework spawned several algorithms \citep{richtarik20223pc, fatkhullin2021ef21}, including extensions for bidirectional (s2w and w2s) biased compression. \citet{EF21-P} introduced \algname{EF21-P}, which combines biased s2w and unbiased w2s compression to achieve improved complexity bounds in the smooth strongly convex setting. More recently, \citet{gruntkowska2024improving} developed \algname{MARINA-P} for smooth non-convex optimization, leveraging correlated unbiased compressors on the server side to obtain tighter complexity bounds than both \algname{EF21} and \algname{EF21-P}. In parallel, \citet{NonsmoothEF21-P} provided the first non-smooth convergence results for \algname{EF21-P}, though limited to the single-node setting.

In order to express communication complexities, we will further need the following quantities.
\begin{definition}[Expected density]
For the given compression operators $\cQ(x)$ and $\cC(x)$, we define the expected density as 
$\zeta_{\cQ} = \sup_{x \in \mathbb{R}^d} \Exp{\norm{\cQ(x)}_0}$ and $\zeta_{\cC} = \sup_{x \in \mathbb{R}^d} \Exp{\norm{\cC(x)}_0}$, 
where $ \norm{y}_0 $ is the number of non-zero components of $ y \in \mathbb{R}^d $.
\end{definition}
Notice that the expected density is well-defined for any compression operator since $\norm{\cQ(x)}_0 \leq d$ and $\norm{\cC(x)}_0 \leq d$.

\subsubsection{Communication-efficient Federated Methods for Non-smooth Optimization}
The landscape of communication-efficient federated methods for non-smooth optimization remains largely unexplored, with most existing research focusing on smooth objectives or single-node scenarios. We discuss the current state of the field and identify the gaps our work aims to address.
\paragraph{Majority of Results on Distributed Optimization with s2w Compression are for Smooth Functions.}
While there is an abundance of work studying compression techniques to reduce s2w communication cost \citep{zheng2019communication, EF21-P, fatkhullin2021ef21, philippenko2021preserved, liu2020double, gorbunov2020linearly, safaryan2022uncertainty, huang2022lower, horvath2022natural, tang2019doublesqueeze, tyurin20232direction, gruntkowska2024improving}, these studies primarily focus on smooth objectives. In the context of Error Feedback methods, to the best of our knowledge, only two works \citep{karimireddy2019error, NonsmoothEF21-P} offer non-smooth convex guarantees, and these are limited to the single-node regime, which has limited practical interest in federated learning contexts.
\paragraph{Existing Literature on Distributed Subgradient Methods Focuses Primarily on w2s Compression.}
While distributed parallel subgradient methods have been extensively studied, existing results either do not offer compressed communications \citep{nedic2009distributed, kiwiel2001parallel, hishinuma2015parallel, zheng2022parallel}, or restrict analysis to specific compression operators like quantization \citep{xia2023distributed, doan2020fast, doan2018convergence, xia2022convergence, emiola2022quantized}, without covering other notable examples from classes \eqref{eq:contractive_compressor} or \eqref{eq:unbiased_compressor}. Moreover, these works consider only w2s compression, neglecting the s2w direction. To our knowledge, there are no comprehensive results addressing non-smooth distributed optimization with s2w compression.
\paragraph{Adaptive Stepsizes in Non-smooth Settings Lack Distributed Guarantees.}
Recent works on adaptive stepsizes in non-smooth convex optimization \citep{khaled2023dowg, defazio2023and, defazio2024road, mishchenkoprodigy, defazio2023learning} have shown promising practical results. However, these studies primarily focus on single-node scenarios and are not directly applicable to federated learning. Polyak stepsizes \citep{Polyak1987, hazan2019revisiting}, in particular, have gained popularity among theoreticians, but the majority of recent results \citep{loizou2021stochastic, oikonomou2024stochastic, jiang2024adaptive} assume smoothness and are again limited to single-node settings. The few results available for non-smooth convex settings \citep{hazan2019revisiting, schaipp2023stochastic} are also confined to single-node scenarios.
\paragraph{Summary and Our Goal.}
In summary, the intersection of non-smooth optimization, communication efficiency, and federated learning remains underexplored. Our work aims to address this gap by providing the first comprehensive study of distributed non-smooth optimization with server-to-worker compression and support for adaptive stepsizes while maintaining optimal convergence rates.

\begin{table*}[t]
\centering
\footnotesize
\caption{Summary of optimization methods employing worker-to-server (w2s) or server-to-worker (s2w) compression schemes.}
\label{tab:methods}    
\begin{threeparttable}
\begin{tabular}{|c|c c c c c|}
    \hline
    Method & Non-smooth & Distributed & \begin{tabular}{c}Compressed \\ communications \end{tabular} & \begin{tabular}{c}Compression \\ type \end{tabular} & \begin{tabular}{c}Adaptive \\ stepsizes \end{tabular} \\ 
    \hline
    \hline
    \begin{tabular}{c} \algname{EF14}\\ {\citep{karimireddy2019error}} \end{tabular} & \cmark & \xmark & \cmark & w2s & \xmark \\                     
    \hline
    \begin{tabular}{c} \algname{EF21-P}\\ {\citep{NonsmoothEF21-P}} \end{tabular} & \cmark & \xmark & \cmark & s2w & \cmark \\                     
    \hline
    \begin{tabular}{c} \algname{MARINA-P}\\ {\citep{gruntkowska2024improving}} \end{tabular} & \xmark & \cmark & \cmark & s2w & \xmark \\                     
    \hline   
    \begin{tabular}{c} \algname{SM} with Polyak Stepsize\\ {\citep{hazan2019revisiting}} \end{tabular} & \cmark & \xmark & \xmark & - & \cmark \\
    \hline
    \begin{tabular}{c} \algname{SM} with Quantization\\ {\citep{xia2023distributed}} \end{tabular} & \cmark & \cmark & \cmark & w2s & \xmark \\
    \hline
    \rowcolor{lightblue} 
    \begin{tabular}{c} \algname{EF21-P}\\ {[OURS]} \end{tabular} & \cmark & \cmark & \cmark & s2w & \cmark \\                 
    \hline
    \rowcolor{lightblue}
    \begin{tabular}{c} \algname{MARINA-P}\\ {[OURS]} \end{tabular} & \cmark & \cmark & \cmark & s2w & \cmark \\                    
    \hline          
\end{tabular}{}
\end{threeparttable}
\end{table*}

\section{Contributions}
We now summarize our main contributions:

{$\bullet$ \bf Extension of \algname{EF21-P} to distributed non-smooth settings.} 
We extend the theory of \algname{EF21-P}, originally developed for single-node scenarios \citep{NonsmoothEF21-P}, to the distributed setting. We prove optimal rates of $\cO\rb{\nfr{1}{\sqrt{T}}}$ for Polyak and constant stepsizes, and a suboptimal rate of $\cO\rb{\nfr{\log T}{\sqrt{T}}}$ for decreasing stepsizes, while establishing communication complexity bounds that match those of classical distributed subgradient methods. This addresses a crucial gap in the theoretical understanding of Error Feedback methods in non-smooth distributed optimization.

{$\bullet$ \bf Introduction of \algname{MARINA-P} for non-smooth objectives.} 
Building upon the recent work of \citet{gruntkowska2024improving}, we extend the applicability of \algname{MARINA-P} beyond smooth non-convex problems to non-smooth convex settings. We establish optimal rates of $\cO\rb{\nfr{1}{\sqrt{T}}}$ for Polyak and constant stepsizes, and a suboptimal rate of $\cO\rb{\nfr{\log T}{\sqrt{T}}}$ for decreasing stepsizes.

{$\bullet$ \bf Superior performance of \algname{MARINA-P} with correlated compressors.} 
Through our empirical studies, we demonstrate that \algname{MARINA-P}, when used with correlated compressors, outperforms \algname{EF21-P} in the non-smooth regime. This result extends the superiority of correlated compressors, previously established for smooth non-convex problems, to non-smooth convex optimization, providing with efficient tools for handling non-smooth objectives in federated settings.

{$\bullet$ \bf Support for diverse stepsize schedules.} 
We provide theoretical guarantees for both \algname{EF21-P} and \algname{MARINA-P} with constant, decreasing, and Polyak stepsizes. This contribution bridges the gap between theoretical advances and practical deep learning scenarios, where adaptive learning rates are commonplace, while maintaining optimal convergence rates.

To the best of our knowledge, our work presents the first theoretical results for distributed non-smooth optimization incorporating s2w compression and adaptive stepsizes, while achieving provably optimal convergence rates.

\section{EF21-P}\label{sec:ef21-p}
We now present the first major contribution of our paper: a distributed version of \algname{EF21-P}for the non-smooth setting. 

Let us first recap the standard single-node \algname{EF21-P} algorithm, which aims to solve \eqref{eq:finite_sum} via the iterative process:
\begin{eqnarray}\label{eq:ef21-smooth}
\squeeze
x^{t+1} &=& x^t-\gamma_t \nabla f(w^t)\\
w^{t+1} &=& w^t+\mathcal{C}^t\left(x^{t+1}-w^t\right), \nonumber
\end{eqnarray}
where $\gamma_t>0$ is a stepsize, $x^0 \in \mathbb{R}^d$ is the initial iterate, $w^0=x^0 \in \mathbb{R}^d$ is the initial iterate shift, and $\mathcal{C}^t$ is an instantiation of a randomized contractive compressor $\mathcal{C}$ sampled at time $t$. 
This method was proposed as a primal\footnote{Since it operates in the primal space of model parameters} counterpart to the standard \algname{EF21}. It has proven particularly useful in bidirectional settings where primal compression is performed on the server side, allowing for the decoupling of primal and dual compression parameter constants. For more details, we refer the reader to the original paper \citep{EF21-P}.
\citet{NonsmoothEF21-P} first extended \algname{EF21-P} to the non-smooth setting. Their key modification was replacing the "smooth" update step with a "non-smooth" one:
\begin{eqnarray}\label{eq:ef21-nonsmooth}
\squeeze
x^{t+1} = x^t- \gamma_t \partial f(w^t). 
\end{eqnarray}
They proved an optimal rates of $\cO\rb{\nfr{1}{\sqrt{T}}}$ for Polyak and constant stepsizes, and a suboptimal rate of $\cO\rb{\nfr{\log T}{\sqrt{T}}}$ for decreasing stepsizes, but only for the single-node regime.
In Algorithm \ref{alg:EF21-P}, we extend these results to the distributed setting, allowing for parallel computation of subgradients $\partial f(w^t)$.

\begin{algorithm}[H]
    \caption{\algname{EF21-P} (distributed version)}\label{alg:EF21-P}
    \begin{algorithmic}[1]
        \State \textbf{Input:} initial points $w^0 = x^0 \in \mathbb{R}^d$, stepsize $\gamma_0 > 0$
        \For{$t = 0, 1, 2, \ldots, T$}
        \For{$i = 1, \ldots, n$ \textbf{on Workers in parallel}}
        \State Receive compressed difference $\Delta^t$ from server
        \State Compute local subgradient $g_i^t = \partial f_i(w^t)$ and send it to server
        \EndFor
        \State \textbf{On Server:}
        \State Receive $g_i^t$ from workers
        \State Choose stepsize $\gamma_t$ (can be set according to \eqref{eq:EF21-P-D:constant_stepsize}, \eqref{eq:gamma_t_polyak}, or \eqref{eq:EF21-P:decr_stepsize})
        \State $x^{t+1} = x^t - \gamma_t \suminn g_i^t$
        \State Compute $\Delta^{t+1} = \mathcal{C}(x^{t+1} - w^t)$ and broadcast it to workers
        \State $w^{t+1} = w^{t} + \Delta^{t+1}$
        \For{$i = 1, \ldots, n$ \textbf{on Workers in parallel}}
        \State $w^{t+1} = w^{t} + \Delta^{t+1}$ 
        \EndFor
        \EndFor
        \State \textbf{Output:} $x^T$
    \end{algorithmic}
\end{algorithm}

At each iteration of distributed \algname{EF21-P}, the workers calculate $\partial f_i(w^t)$ and transmit it to the server. The server then averages the subgradients and updates the global model $x^t$. Subsequently, it computes the compressed difference $\Delta^{t+1} = \cC_i^t(x^{t+1}-w^t)$ and broadcasts the same vector $\Delta^{t+1}$ to all workers. Both the server and workers then use $\Delta^{t+1}$ to update their internal states $w^t$. 
Note that this procedure ensures that the states $w^t$ remain synchronized between workers and the server.

We now present the convergence result for our distributed \algname{EF21-P} algorithm.

\begin{theorem}\label{thm:main:EF21-P-D}
Let Assumptions \ref{as:existence_of_minimizer}, \ref{as:fi_convexity} and \ref{as:fi_lipschitzness} hold.
Define a Lyapunov function
$V^t \eqdef \sqnorm{\xt - x^*} + \frac{1}{\lambda_* \theta} \sqnorm{w^t - \xt}$,
where $\lambda_* \eqdef \fr{\sqrt{1-\alpha}} {1 - \sqrt{1-\alpha}}$ and $\theta \eqdef 1 - \sqrt{1-\alpha}$.
Define also a constant 
$B_* \eqdef 1 + 2\fr{\sqrt{1-\alpha}} {1 - \sqrt{1-\alpha}}$.
Let $\cb{w^t}$ be the sequence produced by \algname{EF21-P} (Algorithm \ref{alg:EF21-P}). Define $\avg{w}^T \eqdef \frac{1}{T} \sum_{t=0}^{T-1} w^t$ and $\what{w}^T \eqdef \frac{1}{\sum_{t=0}^{T-1} \gamma_t}\sum_{t=0}^{T-1} \gamma_t w^t$.

\textbf{1. (Constant stepsize).} If $\gamma_t \eqdef \gamma >0$, then
\begin{eqnarray}\label{eq:EF21-P-D:constant_stepsize_rate}
\squeeze
\Exp{f(\avg{w}^T) - f(x^*)} \le \frac{V^0}{2\gamma T} + \frac{B_* L_0^2 \gamma}{2}.
\end{eqnarray}
If, moreover, optimal $\gamma$ is chosen i.e. 
\begin{eqnarray}\label{eq:EF21-P-D:constant_stepsize}
\squeeze
\gamma \eqdef \fr{1}{\sqrt{T}} \sqrt{\fr{V^0}{B_* L_0^2}},
\end{eqnarray}
 then 
\begin{eqnarray}
\squeeze
\Exp{f(\avg{w}^T) - f(x^*)} \le \frac{\sqrt{B_* L_0^2 V^0}}{\sqrt{T}}.
\end{eqnarray}

\textbf{2. Polyak stepsize.}
If $\gamma_t$ is chosen as
\begin{eqnarray}\label{eq:gamma_t_polyak}
\squeeze
\gamma_t &\eqdef& \fr{f(w^t) - f(x^*)}{B_* \sqnorm{\pf{w^t}}},
\end{eqnarray}
then
\begin{eqnarray}\label{eq:EF21-P-D:polyak_stepsize_rate}
\squeeze
\Exp{f(\avg{w}^T) - f(x^*)} \le \frac{\sqrt{B_* L_0^2 V^0}}{\sqrt{T}}.
\end{eqnarray}

\textbf{3. (Decreasing stepsize).} If $\gamma_t$ is chosen as
\begin{eqnarray}\label{eq:EF21-P:decr_stepsize} 
\gamma_t \eqdef \fr{\gamma_0}{\sqrt{t+1}},
\end{eqnarray}
 then
\begin{eqnarray}
\squeeze
\Exp{f(\what{w}^T) - f(x^*)} \leq \frac{V^0 + 2\gamma_0^2 B_* L_0^2 \log(T + 1)}{\gamma_0\sqrt{T}}.
\end{eqnarray}
If, moreover, optimal $\gamma_0$ is chosen i.e.
\begin{eqnarray} 
\gamma_0 \eqdef \sqrt{\fr{V_0}{2B_* L_0^2 \log(T+1)}},
\end{eqnarray} then 
\begin{eqnarray}
\squeeze
\Exp{f(\what{w}^T) - f(x^*)} \leq 2\sqrt{2B_* L_0^2 V_0}\sqrt{\frac{\log(T+1)}{T}}.
\end{eqnarray}

\end{theorem}
Let us analyze the obtained results. The constant $B_* \eqdef 1 + 2\fr{\sqrt{1-\alpha}} {1 - \sqrt{1-\alpha}} \le \fr{4}{\alpha} -1$ is a decreasing function in $\alpha$, which aligns with intuition since larger values of $\alpha$ correspond to less aggressive compression regimes.
For both constant \eqref{eq:EF21-P-D:constant_stepsize} and Polyak \eqref{eq:gamma_t_polyak} stepsizes, we achieve the optimal rate of $\cO\rb{\nfr{1}{\sqrt{T}}}$ known for uncompressed subgradient methods \citep{Nesterov2013, NIPS2015_7fec306d}. However, achieving this rate requires either knowing the total number of iterations $T$ in advance (for constant stepsize) or knowing the optimal value $f(x^*)$ (for Polyak stepsize), which may be impractical in many applications.
When neither $T$ nor $f(x^*)$ is known, one can employ the decreasing stepsize strategy \eqref{eq:EF21-P:decr_stepsize}. This approach leads to a suboptimal convergence rate of $\cO\rb{\nfr{\log T}{\sqrt{T}}}$ -- a well-known limitation of subgradient methods \citep{Nesterov2013, NonsmoothEF21-P}. 

For both constant and Polyak stepsizes, the following corollary provides explicit complexity bounds, characterizing both the number of iterations and the total communication cost needed to obtain an $\varepsilon$-approximate solution.
\begin{corollary}\label{cor:main:EF21-P-D}
Let the conditions of the Theorem \ref{thm:main:EF21-P-D} are met.
If $\gamma_t$ is set according to \eqref{eq:EF21-P-D:constant_stepsize} or \eqref{eq:gamma_t_polyak} (constant or Polyak stepsizes) 
then \algname{EF21-P} (Algorithm \ref{alg:EF21-P}) requires
\begin{eqnarray}
T = \cO\rb{\frac{L_0^2 R_0^2}{\alpha\eps^2}}
\end{eqnarray}
iterations/communication rounds in order to achieve $\Exp{f(\avg{w}^T) - f(x^*)} \le \eps$
, and the expected total communication cost per worker is $\cO\left(d+\zeta_{\cC}T\right)$.
\end{corollary}
Let us analyze the implications of Corollary \ref{cor:main:EF21-P-D}. In the uncompressed case ($\alpha=1$), our algorithm achieves the optimal rate of standard Subgradient Methods (\algname{SM}) \citep{Nesterov2013} for first-order non-smooth optimization. With Top$K$ compression ($\zeta_{\cC} = K$), the communication complexity becomes $\cO\rb{\nfr{dL_0^2R_0^2}{\varepsilon^2}}$, matching the worst-case complexity of distributed \algname{SM}. This indicates that \algname{EF21-P} with Top$K$ compression cannot improve upon \algname{SM}'s complexity regardless of the compression parameter $\alpha$ -- a fundamental limitation in communication-compressed non-smooth optimization. Our findings align with \citet{balkanski2018parallelization}, who demonstrated that parallelization provides no worst-case benefits for non-smooth optimization.

From a practical perspective, \algname{EF21-P}'s main limitation stems from broadcasting identical compressed differences $\Delta_t$ to all workers, potentially leading to poor approximations of $x^{t+1}$ by $w^t + \Delta_t$. The \algname{MARINA-P} algorithm \citep{gruntkowska2024improving}, originally developed for smooth non-convex problems, addresses this limitation. In the following section, we extend \algname{MARINA-P} to the non-smooth setting.
\section{MARINA-P}\label{subsubsec:marina-p}

Building upon the foundations of the standard \algname{MARINA} algorithm \citep{MARINA, szlendakpermutation}, \citet{gruntkowska2024improving} introduced \algname{MARINA-P}, a primal counterpart designed to operate in the model parameter space. This section presents an extension of \algname{MARINA-P} to the non-smooth convex setting.
\begin{algorithm}[H]
\caption{\algname{MARINA-P}}\label{alg:MARINA-P}
\begin{algorithmic}[1]
\State \textbf{Input:} initial point $x^0 \in \mathbb{R}^d$, initial model shifts $w_i^0 = x^0$ for all $i \in [n]$, stepsize $\gamma_0 > 0$, probability $0 < p \leq 1$, compressors $\mathcal{Q}_i^t \in \mathbb{U}(\omega)$ for all $i \in [n]$
\For{$t = 0, 1, \ldots, T$}
    \For{$i = 1, \ldots, n$ \textbf{on Workers in parallel}}
        \State Compute local subgradient $g_i^t = \partial f_i(w_i^t)$ and send it to server
    \EndFor
    \State \textbf{On Server:}

    \State Receive $g_i^t$ from workers
    \State Choose stepsize $\gamma_t$ (can be set according to \eqref{eq:MARINA-P:constant_stepsize}, \eqref{eq:marina_polyak_type}, or \eqref{eq:marina_decr_type})
    \State $x^{t+1} = x^t - \gamma_t  \suminn g_i^t$ \label{eq:alg:marina_p:update}
    \State Sample $c^t \sim \text{Bernoulli}(p)$
    \If{$c^t = 0$}
        \State Send $\mathcal{Q}_i^t(x^{t+1} - x^t)$ to worker $i$ for $i \in [n]$
    \Else
        \State Send $x^{t+1}$ to all workers
    \EndIf
    \For{$i = 1, \ldots, n$ \textbf{on Workers in parallel}}
        \State $w_i^{t+1} = \begin{cases} 
            x^{t+1} & \text{if } c^t = 1 \\
            w_i^t + \mathcal{Q}_i^t(x^{t+1} - x^t) & \text{if } c^t = 0
        \end{cases}$\label{eq:wtpo_update}
    \EndFor
\EndFor
\State \textbf{Output:} $x^T$
\end{algorithmic}
\end{algorithm}
At each iteration, workers compute subgradients $\p f_i(w_i^t)$ and transmit them to the server. The server aggregates these subgradients and updates the global model $x^t$. With probability $p$ (typically small), the server sends the uncompressed updated model $x^{t+1}$ to all workers. Otherwise, each worker $i$ receives a compressed vector $\cQ_i^t(x^{t+1} - x^t)$. Workers then update their local models $w_i^{t+1}$ based on the received information.
A key feature of \algname{MARINA-P} is that the compressed vectors $\cQ_1^t(x^{t+1} - x^t), \ldots, \cQ_n^t(x^{t+1} - x^t)$ can differ across workers. This distinction is crucial for the algorithm's practical superiority, as it allows for potentially better approximations of $x^{t+1}$ compared to methods like \algname{EF21-P}, especially when the compressors $\cQ_1, \ldots, \cQ_n$ are correlated.

We now present the main convergence results for \algname{MARINA-P} in the non-smooth convex setting.
\begin{theorem}\label{thm:main:MARINA-P-D}
Let Assumptions \ref{as:existence_of_minimizer}, \ref{as:fi_convexity} and \ref{as:fi_lipschitzness} hold. Define a Lyapunov function
$V^t \eqdef \sqnorm{\xt - x^*} + \frac{1}{\lambda_* p} \suminn \sqnorm{w_i^t - \xt}$,
where $\lambda_* \eqdef \fr{\avg{L}_0}{\wt{L}_0}\sqrt{\fr{(1-p)\omega}{p}}$.
Define also a constant 
$\wt{B}_* \eqdef \avg{L}_0^2 + 2{\avg{L}_0}{\wt{L}_0}\sqrt{\fr{(1-p)\omega}{p}}$.
Let $\cb{w_i^t}$ be the sequence produced by \algname{MARINA-P}(Algorithm \ref{alg:MARINA-P}). 
Define $\avg{w}_i^T \eqdef \frac{1}{T} \sum_{t=0}^{T-1} w_i^t$ and $\what{w}_i^T \eqdef \frac{1}{\sum_{t=0}^{T-1} \gamma_t}\sum_{t=0}^{T-1} \gamma_t w_i^t$ for all $i \in [n]$.

\textbf{1. (Constant stepsize).} If $\gamma_t \eqdef \gamma >0$, then
\begin{eqnarray}\label{eq:main:MARINA-P-D:constant_stepsize_rate}
\squeeze
\Exp{\suminn f_i(\avg{w}_i^T) - f(x^*)} \le \frac{V^0}{2\gamma T} + \frac{\wt{B}_* \gamma}{2}.
\end{eqnarray}
If, moreover, optimal $\gamma$ is chosen i.e. 
\begin{eqnarray}\label{eq:MARINA-P:constant_stepsize}
\squeeze
\gamma \eqdef \fr{1}{\sqrt{T}} \sqrt{\fr{V^0}{\wt{B}_*}},
\end{eqnarray}
 then 
\begin{eqnarray}\label{eq:main:cor:MARINA-P-D:constant_stepsize_rate}
\squeeze
\Exp{\suminn f_i(\avg{w}_i^T) - f(x^*)} \le \frac{\sqrt{\wt{B}_* V^0}}{\sqrt{T}}.
\end{eqnarray}

\textbf{2. Polyak stepsize.}
If $\gamma_t$ is chosen as
\begin{eqnarray}\label{eq:marina_polyak_type}
\gamma_t \eqdef \fr{\suminn f_i(w_i^t) - f(x^*)}{\sqnorm{\suminn \pfi{w_i^t}} \rb{1 + 2\fr{\sqrt{\suminn \twonorm{\pfi{w_i^t}}^2}}{\twonorm{\suminn \pfi{w_i^t}}}  \sqrt{\fr{(1-p)\omega}{p}}} } 
,
\end{eqnarray}
then
\begin{eqnarray}\label{eq:main:MARINA-P-D:polyak_stepsize_rate}
\squeeze
\Exp{\suminn f_i(\avg{w}_i^T) - f(x^*)} \le \frac{\sqrt{\wt{B}_* V^0}}{\sqrt{T}}.
\end{eqnarray}
\textbf{3. (Decreasing stepsize).} If $\gamma_t$ is chosen as 
\begin{eqnarray}\label{eq:marina_decr_type}
\gamma_t \eqdef \fr{\gamma_0}{\sqrt{t+1}},
\end{eqnarray}
 then
\begin{eqnarray}
\squeeze
\Exp{\suminn f_i(\what{w}_i^T) - f(x^*)} \leq \frac{V^0 + 2\gamma_0^2 \wt{B}_* \log(T + 1)}{\gamma_0\sqrt{T}}.
\end{eqnarray}
If, moreover, optimal $\gamma_0$ is chosen i.e. 
\begin{eqnarray}
\gamma_0 \eqdef \sqrt{\fr{V_0}{2\wt{B}_* \log(T+1)}},
\end{eqnarray}
 then 
\begin{eqnarray}
\squeeze
\Exp{\suminn f_i(\what{w}_i^T) - f(x^*)} \leq 2\sqrt{2\wt{B}_* V_0}\sqrt{\frac{\log(T+1)}{T}}.
\end{eqnarray}

\end{theorem}

\begin{remark}
For both \algname{EF21-P} and \algname{MARINA-P}, the Polyak stepsize can be efficiently implemented in the distributed setting without additional per-iteration communication overhead. This is because the subgradient values $\pfi{w^t}$ (for \algname{EF21-P}) and $\pfi{w_i^t}$ (for \algname{MARINA-P}) are already computed by the clients and transmitted to the server as part of the algorithm's regular operations.
\end{remark}
Let us analyze these results. The constant $\wt{B}_* \eqdef \avg{L}_0^2 + 2{\avg{L}_0}{\wt{L}_0}\sqrt{\fr{(1-p)\omega}{p}}$ depends on both compression parameters $\omega$ and $p$. Smaller values of $\omega$ correspond to less aggressive compression, while larger values of $p$ indicate more frequent uncompressed communication -- both cases lead to smaller $\wt{B}_*$ and consequently faster convergence.
For both constant \eqref{eq:MARINA-P:constant_stepsize} and Polyak \eqref{eq:marina_polyak_type} stepsizes, we obtain the optimal rate of $\cO\rb{\nfr{1}{\sqrt{T}}}$ \citep{Nesterov2013, NIPS2015_7fec306d}. As with \algname{EF21-P}, achieving this rate requires either knowing the total iterations $T$ (for constant stepsize) or the optimal value $f(x^*)$ (for Polyak stepsize) in advance. 
When such knowledge is unavailable, the decreasing stepsize strategy offers a practical alternative, though it results in a suboptimal $\cO\rb{\nfr{\log T}{\sqrt{T}}}$ convergence rate -- a characteristic limitation of subgradient methods \citep{Nesterov2013}.
It is worth noting that implementing the Polyak stepsize only requires an estimate of $f(x^*)$, rather than knowledge of the Lipschitz constant $L_0$. This characteristic is common among Polyak stepsizes \citep{loizou2021stochastic}.

For the constant and Polyak stepsize regimes, the following corollary establishes complexity bounds and characterizes the communication costs required to achieve an $\varepsilon$-approximate solution.
\begin{corollary}\label{cor:main:MARINA-P-D}
Let the conditions of the Theorem \ref{thm:main:MARINA-P-D} are met and $p = \nfr{\zeta_{\cQ}}{d}$.
If $\gamma_t$ is set according to \eqref{eq:MARINA-P:constant_stepsize} or \eqref{eq:marina_polyak_type} (constant or Polyak stepsizes) 
then \algname{MARINA-P} (Algorithm \ref{alg:MARINA-P}) requires
\begin{eqnarray}
T = \cO\rb{\frac{R^2_0}{\eps^2}\rb{\avg{L}_0^2 + {\avg{L}_0}{\wt{L}_0}\sqrt{\omega\rb{\fr{d}{\zeta_{\cQ}}-1}}}}
\end{eqnarray}
iterations/communication rounds in order to achieve $\Exp{\suminn f_i(\avg{w}_i^T) - f(x^*)} \le \eps$
, and the expected total communication cost per worker is $\cO\left(d+\zeta_{\cQ}T\right)$.
\end{corollary}
This corollary reveals several important properties. With Rand$K$ compression ($\zeta_{\cQ}=K$, $\omega = \nfr{d}{K} - 1$) \citep{beznosikov2023biased}, \algname{MARINA-P} achieves communication complexity $\cO\rb{\nfr{d \wt{L}_0^2R_0^2}{\eps^2}}$. Under the condition $\wt{L}_0^2 = \cO\rb{L_0}$, this matches the optimal per-worker complexity of standard \algname{SM}, up to constant factors \citep{Nesterov2013}.
A notable feature of our complexity result is its independence from the number of workers $n$ in the non-smooth setting -- a known phenomenon in subgradient methods \citep{NIPS2015_7fec306d, balkanski2018parallelization}. This contrasts with \algname{MARINA-P}'s behavior in smooth non-convex settings \citep{gruntkowska2024improving}, where complexity scales as $\cO(\frac{1}{n})$. The absence of theoretical bounds predicting such scaling behavior in non-smooth distributed settings presents an interesting direction for future research.

\algname{MARINA-P}'s primary advantage over \algname{EF21-P} lies in its ability to employ worker-specific compression operators $\cQ_i$, enabling more accurate approximations of the global model, particularly when using correlated compressors. The following section examines various constructions of $\cQ_i$ that leverage this flexibility to enhance practical performance.

\subsection{Three Ways to Compress: A Recap}\label{subsec:three_compressors}

In our experiments, we will examine three distinct approaches to constructing the compressors $\{\cQ_i\}$ in \algname{MARINA-P}, as outlined in \citep{gruntkowska2024improving}:

\textbf{1. Same Compressor.} The conventional method where the server broadcasts an identical compressed message to all workers. Using a single Rand$K$ compressor $\cQ$, we have $\cQ_1^t(x^{t+1}-x^t) = \dots =\cQ_n^t(x^{t+1}-x^t)=\cQ^t(x^{t+1}-x^t)$ for all workers $i \in [n]$. This approach, while simple, limits the amount of information conveyed.

\textbf{2. Independent Compressors.} This strategy employs a set of independent Rand$K$ compressors $\cQ_i, i \in [n]$, generating distinct, independent sparse vectors $\cQ_1(x), \ldots, \cQ_n(x)$ for input $x \in \R^d$. This method allows for more diverse information transmission but lacks coordination between compressors.

\textbf{3. Correlated Compressors.} Introduced by \citet{szlendakpermutation}, this approach uses coordinated compressors, with Perm$K$ being a key example. For $d \geq n$ and $d = qn$, $q \in \mathbb{N}_{>0}$, Perm$K$ is defined as:
\begin{definition}[Perm$K$]
Let $\pi = (\pi_1,\dots,\pi_d)$ be a random permutation of $\{1, \dots, d\}$. For $x \in \R^d$ and $i \in [n]$:
\begin{eqnarray}
\cQ_i(x) \eqdef n \times \sum \limits_{j = q (i - 1) + 1}^{q i} x_{\pi_j} e_{\pi_j}.
\end{eqnarray}
\end{definition}
Perm$K$ ensures $\frac{1}{n} \sum_{i=1}^n \cQ_i(x) = x$ deterministically, offering superior theoretical and practical performance in \algname{MARINA-P} compared to the other two approaches. This method exploits correlation between compressors to achieve better approximation of the global model and improved communication efficiency.
\begin{figure*}[t]
\centering
\includegraphics[width=0.9\textwidth]{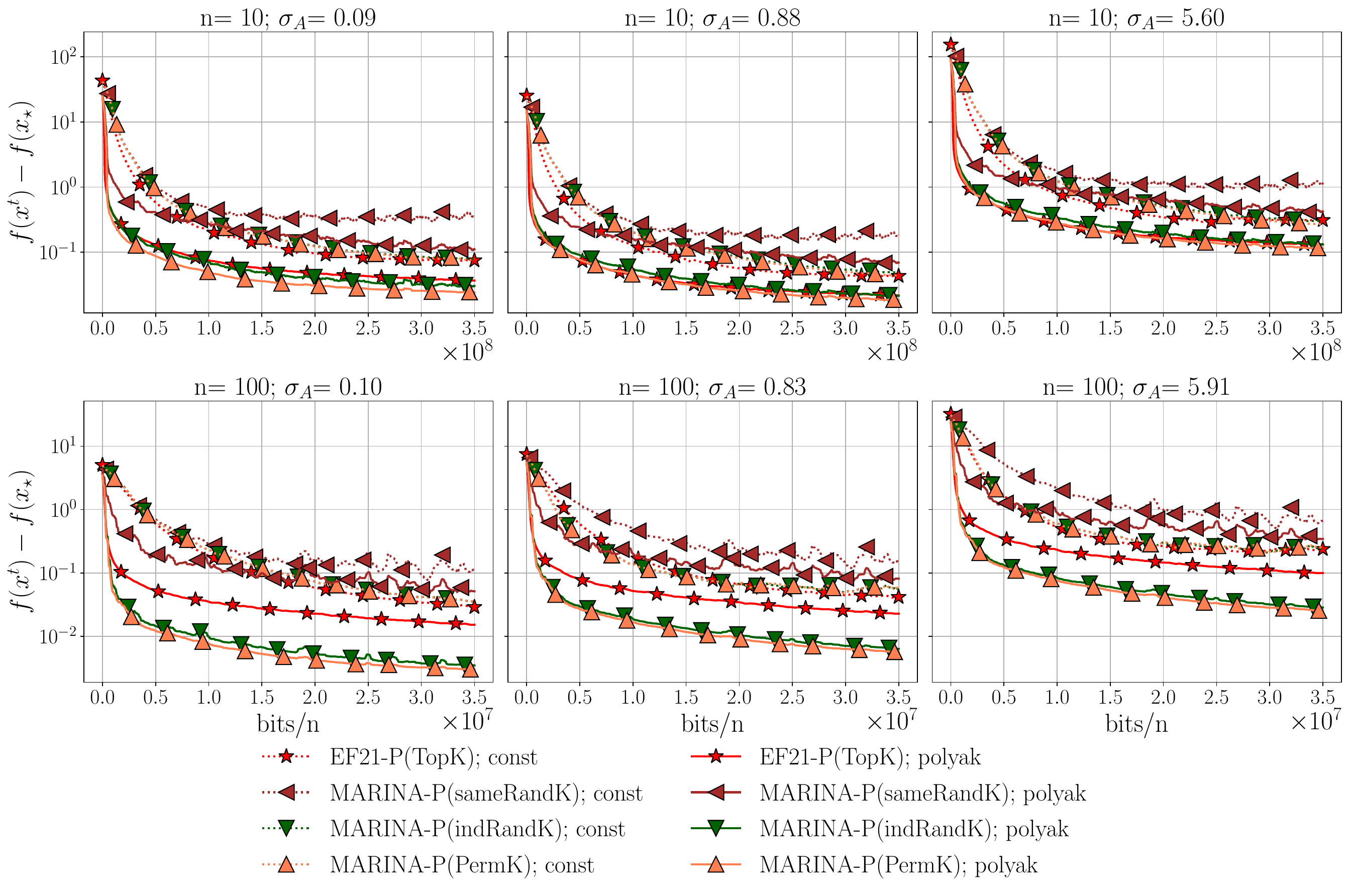}
\caption{Performance comparison of \algname{EF21-P} with Top$K$ and \algname{MARINA-P} with sameRand$K$, indRand$K$, and Perm$K$ compressors ($K = \nicefrac{d}{n}$). The left column of the legend corresponds to experiments with constant stepsizes, while the right column shows results with Polyak stepsizes. All stepsizes were set to the largest theoretically acceptable value multiplied by an individually tuned constant factor, selected from the set $\{2^{-9}, 2^{-8}, \dots, 2^{7}\}$.}
\label{fig:main_marina_ef21}
\end{figure*}
\section{Experiments}\label{sec:exps}
To verify our theoretical results, we conducted experiments comparing \algname{MARINA-P} with different compressor configurations from Section \ref{subsec:three_compressors} (sameRand$K$, indRand$K$, and Perm$K$)  against \algname{EF21-P} with Top$K$ compression. We consider a synthetic non-smooth convex finite sum function $f(x) = \suminn \onenorm{\bA_i x}$, where $\bA_i \in \mathbb{R}^{d \times d}$. We set the dimension $d = 1000$  and tested scenarios with different numbers of nodes $n \in \{10, 100\}$ and different data heterogeneity regimes, controlled by the data dissimilarity measure
\begin{eqnarray}
\sigma_A \eqdef \sqrt{\suminn \sqnorm{\bA_i} - \rb{\sumjnn \twonorm{\bA_j}}^2 }.
\end{eqnarray}
For all configurations, we set $K = \nicefrac{d}{n}$ to ensure a fair comparison of communication costs.
Figure \ref{fig:main_marina_ef21} presents our results, showing that \algname{MARINA-P} with correlated compressors (indRand$K$ and Perm$K$) consistently outperforms other configurations across different node counts and data heterogeneity levels. This performance advantage is particularly pronounced for a large number of clients (e.g., $n = 100$).
For a more detailed description of the experimental setup, we refer readers to Appendix \ref{sec:experiments_extra}.

\section{Conclusion and Future Directions}
In this paper, we have presented a comprehensive analysis of distributed non-smooth optimization with server-to-worker compression. We extended \algname{EF21-P} to the distributed setting and introduced a non-smooth version of \algname{MARINA-P}, providing theoretical guarantees for both algorithms under constant, decreasing, and Polyak stepsizes. To the best of our knowledge, this work presents the first theoretical results for distributed non-smooth optimization that incorporate server-to-worker compression and adaptive stepsizes. Our empirical studies demonstrate the superior performance of \algname{MARINA-P} with correlated compressors in non-smooth settings. 

While our work advances non-smooth federated optimization with server-to-worker compression, several important avenues remain unexplored. Future research could address worker-to-server compression or bidirectional compression schemes. Additionally, incorporating local update steps \citep{demidovichmethods} into our framework could further reduce communication overhead. 

\section*{Acknowledgements}
The research reported in this publication was supported by funding from King Abdullah University of Science and Technology (KAUST): i) KAUST Baseline Research Scheme, ii) Center of Excellence for Generative AI, under award number 5940, iii) SDAIA-KAUST Center of Excellence in Artificial Intelligence and Data Science
\newpage
\bibliographystyle{plainnat}
\bibliography{arxiv_marina-n}

\newpage


\appendix

\part*{APPENDIX}

\section{Experiments: Missing Details and Extra Results}\label{sec:experiments_extra}

In this section, we provide missing details on the experimental setting from Section \ref{sec:exps}. The source code is available in the following GitHub repository: \url{https://anonymous.4open.science/r/MARINA-P_project_source_code-670F/}.

To verify our theoretical results, we conducted experiments comparing \algname{MARINA-P} with different compressor configurations from Section \ref{subsec:three_compressors} (sameRand$K$, indRand$K$, and Perm$K$)  against \algname{EF21-P} with Top$K$ compression. 

\paragraph{Hardware and Software.} All algorithms were implemented in Python 3.10. We utilized three different CPU cluster node types:
\begin{enumerate}
\item AMD EPYC 7702 64-Core;
\item Intel(R) Xeon(R) Gold 6148 CPU @ 2.40GHz;
\item Intel(R) Xeon(R) Gold 6248 CPU @ 2.50GHz.
\end{enumerate}

\begin{algorithm}[H]
\caption{Synthetic datasets generation routine}
\label{alg:matrix_generation}
\begin{algorithmic}[1]
    \State \textbf{Parameters:} number nodes $n$, dimension $d$, parameter $\mu=10^{-6}$, and noise scale $s$.
    \For{$i = 1, \dots, n$}
    \State Generate random noises $\nu_i^s = 1 + s \xi_i^s $,  i.i.d. $\xi_i^s\sim \mathcal {N}(0, 1)$
    \State Take the initial tridiagonal matrix
    \[\bA_i = \frac{\nu_i^s}{4}\left( \begin{array}{cccc}
        2 & -1 & & 0\\
        -1 & \ddots & \ddots & \\
        & \ddots & \ddots & -1 \\
        0 & & -1 & 2 \end{array} \right) \in \R^{d \times d}\]
    \EndFor
    \State Take the mean of matrices $\bA = \frac{1}{n}\sum_{i=1}^n \bA_i$
    \State Find the minimum eigenvalue $\lambda_{\min}(\bA)$
    \For{$i = 1, \dots, n$}
    \State Update matrix $\bA_i = \bA_i + (\mu - \lambda_{\min}(\bA)) \bI$
    \EndFor
    \State Sample starting point $x^0 \sim \mathcal {N}(0, \bI)$
    \State \textbf{Output:} matrices $\bA_1, \cdots, \bA_n$, starting point $x^0$
\end{algorithmic}
\end{algorithm}

\paragraph{Objective and Datasets.}
The primary goal of these numerical experiments is to illustrate our theoretical findings and motivate further practical comparisons of \algname{MARINA-P} against other baselines.

We consider a finite sum function $f(x) = \suminn f_i(x)$, consisting of synthetic non-smooth convex  
functions 
$$f_i(x) \eqdef \onenorm{\bA_i x},$$
where $\bA_i \in \mathbb{R}^{d \times d}$ and $\bA_i = \bA_i^{\top}$ is the training data that belongs to the device/worker $i$.
This objective was chosen for its simplicity to synthetically emulate the behavior of distributed training and to collect all required theoretical metrics, such as function suboptimality $f(w^t) - f(x^*)$. 
For this function, it is known that $x^* = (0,0,\dots,0)^{\top}$, $f(x^*) = 0$. Each subgradient $\pfi{x}$ can be explicitly written (and computed) as $\pfi{x} = \bA_i^{\top} \sign\rb{\bA_i x}$ (see Example 3.44 of the book \citep{beck2017first} for proof details), where $\sign$ is the componentwise sign operator, i.e. 
\begin{eqnarray}
\sign({x})_i= \begin{cases}1, & x_i \geq 0 \\ -1, & x_i<0\end{cases}.
\end{eqnarray}
Note, that $\pf{x}$ can be computed as  $\pf{x} = \suminn \pfi{x} $.
In all experiments of this section, we have $d = 1000$.

We generated synthetic matrices $\cb{\bA_i}_{i=1}^{n}$ (training data) via Algorithm \ref{alg:matrix_generation}
This data generation routine was inspired by a similar one used for solving synthetic quadratic problems (see Algorithm 11 in \citep{richtarik20223pc}).
However, we introduced several minor modifications to the original algorithm for the needs of this project.
We generated optimization problems having different numbers of nodes $n \in \{10, 100\}$ and different data heterogeneity regimes, controlled by the empirically proposed data dissimilarity measure 
\begin{eqnarray}
\sigma_A \eqdef \sqrt{\suminn \sqnorm{\bA_i} - \rb{\sumjnn \twonorm{\bA_j}}^2 }.
\end{eqnarray}
From the definition, it follows that the case of similar (or even identical) functions $f_i$ relates to the small (or even 0) value of $\sigma_A$, whereas in the case of completely different $f_i$ (which relate to heterogeneous data regime) $\sigma_A$ can be large.
In our experiments, homogeneity of each optimization task is controlled by noise scale $s$ introduced in Algorithm \ref{alg:matrix_generation}. 
Indeed, for the noise scale $s=0$, all matrices $\bA_i$ are equal, whereas with the increase of the noise scale, functions become less "similar" and $\sigma_A$ rises. We take noise scales $s \in\{0.1, 1.0, 10.0\}$. 
Table \ref{tbl:sigma_A_summary} summarizes the $\sigma_A$ values corresponding to these noise scales for $n \in \cb{10, 100}$.

\begin{table}[h]
    \caption{Summary of the data heterogeneity $\sigma_{A}$ values for different number of nodes $n$ and various noise scales $s$.} 
    \label{tbl:sigma_A_summary}
    \centering
    \begin{tabular}{l l l l}
        \toprule
        \diagbox[width=1.2cm ,  height=0.5cm]{ $n$ }{\raisebox{0.5ex}{ $s$}}  & $0.1$ & $1.0$ & $10.0$ \\
        \midrule            
        $10$            & $0.09$  & $0.88$ & $5.60$      \\  
        $100$         & $0.10$ & $0.83$ & $5.91$ \\
        \bottomrule
    \end{tabular}
\end{table}

\paragraph{Baselines and Hyperparameters.}

For each dataset (determined by values $n$ and $\sigma_A$), we run the following baselines:
\begin{enumerate}
\item \algname{EF21-P} with Top$K$ compressor;
\item \algname{MARINA-P} with sameRand$K$ compressor;
\item \algname{MARINA-P} with indRand$K$ compressors;
\item \algname{MARINA-P} with Perm$K$ compressors.
\end{enumerate}
where sameRand$K$, indRand$K$, and Perm$K$ are defined as described in subsection \ref{subsec:three_compressors}.

In all experiments, we set $K = \nicefrac{d}{n}$ and for \algname{MARINA-P} we additionally choose $p=\nicefrac{K}{d}$ to ensure a fair comparison of communication costs. Indeed, whereas for \algname{EF21-P} with Top$K$, parameter $K$ (and therefore $\zeta_{\cQ}=K$) is deterministic and fixed throughout the optimization process, in the case of \algname{MARINA-P}, $K$ is random, but $\zeta_{\cQ}= dp + (1-p)K = d\rb{\nfr{K}{d}} + \rb{1 - \nfr{K}{d}}K \le 2K = \cO(K)$, meaning that the choice of $p = \nfr{K}{d}$ on expectation guarantees similar communication costs for \algname{EF21-P} and \algname{MARINA-P}.

For all algorithms, at each iteration (communication round) we updated the following metrics being tracked throughout the whole optimization procedure:
\begin{enumerate}
\item Function suboptimality $f(x^t) - f(x^*)$;
\item Number of bits per worker send from server to clients (titled as ``bits/n'' on corresponding Figure \ref{fig:main_marina_ef21_appndx}).
\end{enumerate}
We employed 64-bit precision in our experiments. Our communication model assumes that the server transfers $(65 + \log_2(d))q$ bits to each worker, where $q$ represents the number of non-zero entries retained after sparsification. This total is broken down as follows:
\begin{itemize}
    \item 64 bits allocated for each non-zero value;
    \item 1 bit for the sign of each entry;
    \item $\log_2(d)$ bits to encode the position of each non-zero entry.
\end{itemize}
The same communication model was also used in \citep{horvath2022natural}.

For each value of $n \in \{10, 100\}$, we allocated an individual communication budget: $3.5 \cdot 10^8$ bits for $n=10$ and $3.5 \cdot 10^7$ bits for $n=100$. Each algorithm was terminated upon reaching its respective budget.

In experiments utilizing constant stepsizes, we set the stepsize to the largest theoretically acceptable value, multiplied by an individually tuned factor. This factor was selected from the set $\{2^{-9}, 2^{-8}, \dots, 2^{7}\}$. For experiments employing adaptive stepsizes, we similarly tuned a constant factor. During the optimization procedure, this factor was multiplied by the theoretically defined adaptive stepsize at each iteration.
Tables \ref{tab:stepsizes} and \ref{fig:stepsize_factors} summarize the theoretical stepsize formulas and optimal tuned stepsize multiplicative factors.

\begin{table*}[t]
\centering
\footnotesize
\caption{Summary of theoretical stepsize formulas for \algname{EF21-P} and \algname{MARINA-P} algorithms.}
\label{tab:stepsizes}    
\begin{threeparttable}
\begin{tabular}{ l c c c c }
    \toprule
    \diagbox[width=3.0cm,  height=1cm]{Method}{\raisebox{0.8ex}{   Stepsize type}} & Constant & Decreasing & Polyak & Reference \\ 
    \midrule
    \algname{EF21-P} & $\fr{1}{\sqrt{T}} \sqrt{\fr{V^0}{B_* L_0^2}}$ & $\fr{\gamma_0}{\sqrt{t+1}}$ & $\fr{f(w^t) - f(x^*)}{B_* \sqnorm{\pf{w^t}}}$ & \eqref{eq:EF21-P-D:constant_stepsize}, \eqref{eq:gamma_t_polyak} \\
    \algname{MARINA-P} & $\fr{1}{\sqrt{T}} \sqrt{\fr{V^0}{\wt{B}_*}}$ & $\fr{\gamma_0}{\sqrt{t+1}}$ & $\fr{\suminn f_i(w_i^t) - f(x^*)}{\sqnorm{\suminn \pfi{w_i^t}} \rb{1 + 2\fr{\sqrt{\suminn \twonorm{\pfi{w_i^t}}^2}}{\twonorm{\suminn \pfi{w_i^t}}}  \sqrt{\fr{(1-p)\omega}{p}}} }$ & \eqref{eq:MARINA-P:constant_stepsize}, \eqref{eq:marina_polyak_type} \\
    \bottomrule
\end{tabular}
\begin{tablenotes}
    \item[1] For the decreasing stepsize, optimal $\gamma_0 = \sqrt{\fr{V_0}{2B_* L_0^2 \log(T+1)}}$ for \algname{EF21-P} \quad and \quad $\gamma_0 = \sqrt{\fr{V_0}{2\wt{B}_* \log(T+1)}}$ for \algname{{MARINA-P}};
    \item[2] $B_*$, $\wt{B}_*$, and other constants are defined in the respective theorems.
\end{tablenotes}
\end{threeparttable}
\end{table*}

\begin{figure}[H]
\centering
\small
\footnotesize
\begin{minipage}{0.48\textwidth}
\centering
\caption{Constant stepsize; $n = 10$.}
\label{tab:const_stepsize_n10}    
\begin{tabular}{ l l l l }
    \toprule
    \diagbox[width=3.3cm, height=0.8cm]{Method}{\raisebox{0.1ex}{$s$}} & 0.1 & 1.0 & 10.0 \\
    \midrule
    \algname{EF21-P} with Top$K$ & 0.5 & 1.0 & 1.0 \\
    \algname{MARINA-P} sameRand$K$ & 0.03125 & 0.03125 & 0.03125 \\
    \algname{MARINA-P} indRand$K$ & 0.03125 & 0.03125 & 0.03125 \\
    \algname{MARINA-P} Perm$K$ & 0.03125 & 0.03125 & 0.03125 \\
    \bottomrule
\end{tabular}
\end{minipage}
\hfill
\begin{minipage}{0.48\textwidth}
\centering
\caption{Constant stepsize; $n = 100$.}
\label{tab:const_stepsize_n100}    
\begin{tabular}{ l l l l }
    \toprule
    \diagbox[width=3.3cm, height=0.8cm]{Method}{\raisebox{0.1ex}{$s$}} & 0.1 & 1.0 & 10.0 \\ 
    \midrule
    \algname{EF21-P} with Top$K$ & 4.0 & 4.0 & 8.0 \\
    \algname{MARINA-P} sameRand$K$ & 0.03125 & 0.03125 & 0.03125 \\
    \algname{MARINA-P} indRand$K$ & 0.03125 & 0.0625 & 0.0625 \\
    \algname{MARINA-P} Perm$K$ & 0.03125 & 0.0625 & 0.0625 \\
    \bottomrule
\end{tabular}
\end{minipage}

\vspace{1cm}

\begin{minipage}{0.48\textwidth}
\centering
\caption{Polyak stepsize; $n = 10$.}
\label{tab:polyak_stepsize_n10}    
\begin{tabular}{ l l l l }
    \toprule
    \diagbox[width=3.3cm, height=1cm]{Method}{\raisebox{0.5ex}{$s$}} & 0.1 & 1.0 & 10.0 \\ 
    \midrule
    \algname{EF21-P} with Top$K$ & 16.0 & 16.0 & 16.0 \\
    \algname{MARINA-P} sameRand$K$ & 2.0 & 2.0 & 2.0 \\
    \algname{MARINA-P} indRand$K$ & 2.0 & 2.0 & 2.0 \\
    \algname{MARINA-P} Perm$K$ & 2.0 & 2.0 & 2.0 \\
    \bottomrule
\end{tabular}
\end{minipage}
\hfill
\begin{minipage}{0.48\textwidth}
\centering
\caption{Polyak stepsize; $n = 100$.}
\label{tab:polyak_stepsize_n100}    
\begin{tabular}{ l l l l }
    \toprule
    \diagbox[width=3.3cm, height=0.8cm]{Method}{\raisebox{0.1ex}{$s$}} & 0.1 & 1.0 & 10.0 \\ 
    \midrule
    \algname{EF21-P} with Top$K$ & 16.0 & 16.0 & 16.0 \\
    \algname{MARINA-P} sameRand$K$ & 2.0 & 2.0 & 2.0 \\
    \algname{MARINA-P} indRand$K$ & 2.0 & 2.0 & 2.0 \\
    \algname{MARINA-P} Perm$K$ & 2.0 & 2.0 & 2.0 \\
    \bottomrule
\end{tabular}
\end{minipage}
\caption{Optimal stepsize multiplicative factors for different methods, number of nodes, and heterogeneity levels.}
\label{fig:stepsize_factors}
\end{figure}
For \algname{MARINA-P}, we initialize $w^0_i = x^0$ for all $i \in [n]$, and $w^0 = x^0$ for \algname{EF21-P}. This choice results in $V^0 = R^2 = \sqnorm{x^0 - x^*} = \sqnorm{x^0}$, allowing explicit computation of $V^0$ constants in all cases.

In our experiments, we estimated the Lipschitz smoothness constants $L_{0,i}$ as $L_{0,i} \sim \twonorm{\bA_i}$. Although this approximation is not theoretically precise, we adopted it primarily for computational simplicity. Moreover, our constant multiplier tuning process compensates for any inaccuracies in the estimated $L_{0,i}$. It's worth noting that the $L_{0,i} \sim \twonorm{\bA_i}$ estimation is reasonably close to the worst-case bound, as demonstrated by:
$$\twonorm{\pfi{x}} = \twonorm{{\bA_i^{\top}\sign\rb{\bA_i x} }}\le \twonorm{\bA_i^{\top}}\twonorm{\sign\rb{\bA_i x}} \le \twonorm{\bA_i}\sqrt{d}.$$ 
We also defined $L_{0}$ as $L_{0} = \suminn L_{i,0}$.

\begin{figure*}[t]
\centering
\includegraphics[width=0.9\textwidth]{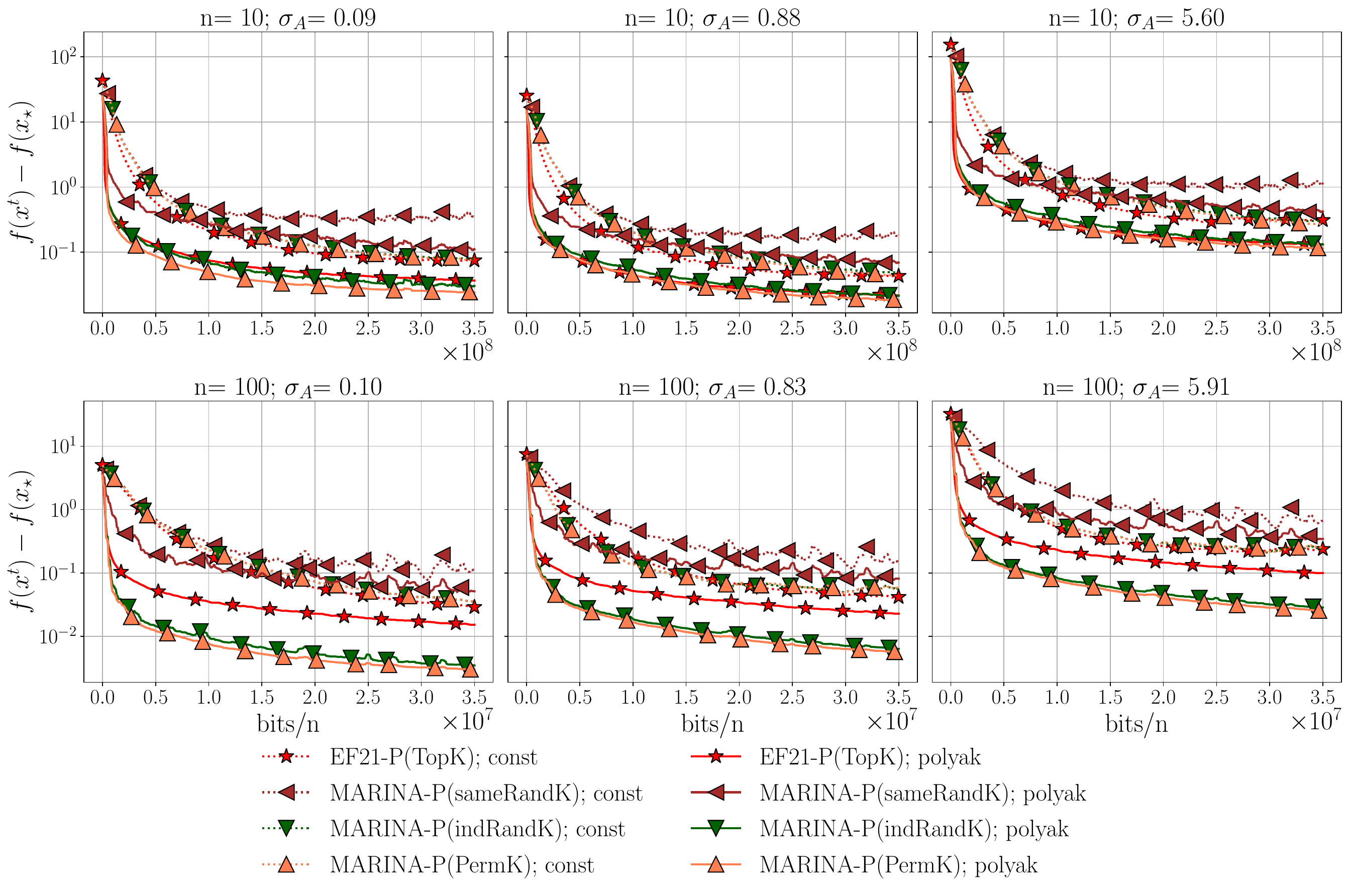}
\caption{Performance comparison of \algname{EF21-P} with Top$K$ and \algname{MARINA-P} with sameRand$K$, indRand$K$, and Perm$K$ compressors ($K = \nicefrac{d}{n}$). The left column of the legend corresponds to experiments with constant stepsizes, while the right column shows results with Polyak stepsizes. All stepsizes were set to the largest theoretically acceptable value multiplied by an individually tuned constant factor, selected from the set $\{2^{-9}, 2^{-8}, \dots, 2^{7}\}$.}
\label{fig:main_marina_ef21_appndx}
\end{figure*}

\paragraph{Comparison of Convergence Behavior.}
We now present a more detailed version of the convergence comparison initially introduced in Section \ref{sec:exps} of the main draft. Our experiments compare the performance of \algname{EF21-P} with Top$K$ and \algname{MARINA-P} with sameRand$K$, indRand$K$, and Perm$K$ compressors across the different datasets described in the previous section.
Figure \ref{fig:main_marina_ef21_appndx} illustrates the following key observations:
\begin{enumerate}
\item \textbf{Superiority of correlated compressors in the non-smooth convex setting.}
For both constant and Polyak stepsizes, \algname{MARINA-P} with Perm$K$ compressors slightly outperforms \algname{MARINA-P} with indRand$K$ compressors, showing significant improvement over the conventional approach using the sameRand$K$ scheme. This observation suggests that correlated compressors indeed ensure better approximation of the compressed difference $\frac{1}{n} \sum_{i=1}^n \cQ_i(\xtpo - \xt) \approx \xtpo - \xt$ (with equality in the case of Perm$K$), leading to superior convergence performance in practice. This behavior aligns well with experiments in the smooth non-convex setting from \citep{gruntkowska2024improving}.

\item \textbf{Superior convergence behavior with adaptive stepsizes.}
Each pair of experiments differing only in stepsize strategy (e.g., \algname{EF21-P} with Top$K$, represented in Figure \ref{fig:main_marina_ef21_appndx} with the same color and marker but different linestyles) demonstrates the practical efficiency of adaptive stepsize schemes. This marks the first time in the literature that such behavior has been experimentally confirmed in the communication-efficient distributed non-smooth convex setting.

\item \textbf{\algname{MARINA-P} with correlated compressors and Polyak stepsize outperforms for all datasets.} 
Figure \ref{fig:main_marina_ef21_appndx} reveals that while all algorithms under constant stepsizes exhibit similar convergence behavior (slightly outperformed by \algname{EF21-P} with Polyak stepsize), \algname{MARINA-P} with correlated compressors and Polyak stepsize demonstrates superior performance. This advantage is particularly pronounced when $n=100$.
\end{enumerate}

These experimental results validate our theoretical findings and highlight the practical advantages of \algname{MARINA-P} with correlated compressors and adaptive stepsizes in the non-smooth convex distributed optimization setting.

\section{Basic Facts and Inequalities}
\textbf{Useful inequalities:}
For all $x, y, x_1, \ldots, x_n \in \mathbb{R}^d$, $s > 0$ and $\alpha \in (0, 1]$, we have:
\begin{eqnarray}
    \twonorm{\suminn x_i} &\le& \suminn \twonorm{x_i},\label{eq:jensen_1} \\
    \sqnorm{\suminn x_i} &\le& \suminn \sqnorm{x_i},\label{eq:jensen} \\
    \twonorm{x + y}^2 &\leq& (1 + s)\twonorm{x}^2 + (1 + s^{-1})\twonorm{y}^2, \label{eq:gen_young} \\
    \twonorm{x + y}^2 &\leq& 2\twonorm{x}^2 + 2\twonorm{y}^2, \label{eq:young}\\
    \langle x, y \rangle &\leq& \frac{\twonorm{x}^2}{2s} + \frac{s\twonorm{y}^2}{2}, \label{eq:langles_gen_young}\\
    (1 - \alpha)\left(1 + \frac{\alpha}{2}\right) &\leq& 1 - \frac{\alpha}{2},\label{eq:gen_young_cons1} \\
    (1 - \alpha)\left(1 + \frac{4}{\alpha}\right) &\leq& \frac{4}{\alpha},\label{eq:gen_young_cons2} \\
    \langle a, b \rangle &=& \frac{1}{2}\left(\twonorm{a}^2 + \twonorm{b}^2 - \twonorm{a - b}^2\right).
\end{eqnarray}

\textbf{Tower property:} For any random variables $X$ and $Y$, we have
\begin{eqnarray}\label{eq:tower_property}
    \Exp{\Exp{X \mid Y}} = \Exp{X}. 
\end{eqnarray}

\textbf{Cauchy-Bunyakovsky-Schwarz inequality:} For any random variables $X$ and $Y$, we have
\begin{eqnarray}\label{eq:cauchy_bunyakovsky_schwarz}
    \abs{\Exp{XY}} \le \sqrt{\Exp{X^2}\Exp{Y^2}}. 
\end{eqnarray}

\textbf{Variance decomposition:} For any random vector $X \in \mathbb{R}^d$ and any non-random $c \in \mathbb{R}^d$, we have
\begin{eqnarray}\label{eq:bvd}
    \Exp{\twonorm{X - c}^2} = \Exp{\twonorm{X - \Exp{X}}^2} + \twonorm{\Exp{X} - c}^2.
\end{eqnarray}

\textbf{Jensen’s inequality:} For any random vector $X \in \mathbb{R}^d$ and any convex function $g:\R^d\mapsto\R$, we have
\begin{eqnarray}\label{eq:jensen_general}
    g(\Exp{X}) \le \Exp{g(X)}.
\end{eqnarray}

\begin{lemma}[Lemma 3 of \citet{richtarik2021ef21}]\label{le:optimal_t-Peter} Let $0<p< 1$ and for $s>0$ let $\theta(s)$ and $\beta(s)$ be defined as 
    \begin{eqnarray}
        \theta(s) &\eqdef& 1 - (1 - p)(1 + s), \qquad
        \beta(s)  \eqdef (1 - p)(1 + s^{-1}).\notag
    \end{eqnarray} Then the solution of the optimization problem
    \begin{eqnarray}\label{eq:98g_(89fd8gf9d} \min_{s} \left\{ \fr{\beta(s)}{\theta(s)} \;:\; 0<s<\fr{p}{1-p}\right\}
    \end{eqnarray}
    is given by $s^* = \fr{1}{\sqrt{1-p}}-1$. Furthermore, $\theta(s^*) = 1-\sqrt{1-p}$, $\beta(s^*) = \fr{1-p}{1-\sqrt{1-p}}$ and
    \begin{eqnarray}\label{eq:n98fhgd98hfd}\sqrt{\fr{\beta(s^*)}{\theta(s^*)}} = \fr{1}{\sqrt{1-p}} -1 = \fr{1}{p} + \fr{\sqrt{1-p}}{p} - 1 \leq \fr{2}{p}-1.
    \end{eqnarray}
    In the trivial case $p = 1$, we have $\fr{\beta(s)}{\theta(s)} = 0$ for any $s > 0$, and \eqref{eq:n98fhgd98hfd} is satisfied.
\end{lemma}

\section{Missing Proofs For \algname{EF21-P}}

In this section, we present the detailed proofs for the theoretical results of \algname{EF21-P} (Algorithm \ref{alg:EF21-P}). Before delving into the proofs, we first discuss how our contribution extends the original results \citep{NonsmoothEF21-P} on \algname{EF21-P} for the non-smooth convex setting.

Recall that the standard single-node \algname{EF21-P} algorithm \citep{EF21-P} in the smooth case takes the form:
\begin{eqnarray}\label{eq:appendix:ef21-smooth}
x^{t+1} &=& x^t-\gamma_t \nabla f(w^t)\\ 
w^{t+1} &=& w^t+\mathcal{C}^t\left(x^{t+1}-w^t\right). \nonumber
\end{eqnarray}

The key modification introduced by \citet{NonsmoothEF21-P} was to replace the "smooth" update step \eqref{eq:appendix:ef21-smooth} with a "non-smooth" one:
\begin{eqnarray}\label{eq:appendix:ef21-nonsmooth}
x^{t+1} = x^t- \gamma_t \partial f(w^t), 
\end{eqnarray}
resulting in Algorithm \ref{alg:EF21-P-S}.

\begin{figure}[H]
\begin{minipage}[t]{0.42\textwidth}
\begin{algorithm}[H]
    \caption{\algname{EF21-P} (single-node version)}\label{alg:EF21-P-S}
    \begin{algorithmic}[1]
        \State \textbf{Input:} initial points $w^0, x^{0}\in \mathbb{R}^d$, stepsize $\gamma_0 > 0$
        \For{$t = 0, 1, 2, \ldots, T$}
            \State Compute subgradient $g^t = \partial f(w^t)$
			\State Choose stepsize $\gamma_t$ (can be set according to \eqref{eq:EF21-P-D:constant_stepsize}, \eqref{eq:gamma_t_polyak}, or \eqref{eq:EF21-P:decr_stepsize})
            \State $x^{t+1} = x^t - \gamma_t  g^t$
            \State Compute $\Delta^{t+1} = \mathcal{C}\left(x^{t+1} - w^t\right)$
            \State $w^{t+1} = w^{t} + \Delta^{t+1}$
        \EndFor
        \State \textbf{Output:} $x^T$
    \end{algorithmic}
\end{algorithm}
\end{minipage}
\hfill
\begin{minipage}[t]{0.54\textwidth}
\begin{algorithm}[H]
    \caption{\algname{EF21-P} (distributed version)}\label{alg:EF21-P-D}
    \begin{algorithmic}[1]
        \State \textbf{Input:} initial points $w^0 = x^0 \in \mathbb{R}^d$, stepsize $\gamma_0 > 0$
        \For{$t = 0, 1, 2, \ldots, T$}
        \For{$i = 1, \ldots, n$ \textbf{on Workers in parallel}}
        \State Receive compressed difference $\Delta^t$ from server
        \State Compute local subgradient $g_i^t = \partial f_i(w^t)$ and send it to server
        \EndFor
        \State \textbf{On Server:}
        \State Receive $g_i^t$ from workers
        \State Choose stepsize $\gamma_t$ (can be set according to \eqref{eq:EF21-P-D:constant_stepsize}, \eqref{eq:gamma_t_polyak}, or \eqref{eq:EF21-P:decr_stepsize})
        \State $x^{t+1} = x^t - \gamma_t \suminn g_i^t$
        \State Compute $\Delta^{t+1} = \mathcal{C}(x^{t+1} - w^t)$ and broadcast it to workers
        \State $w^{t+1} = w^{t} + \Delta^{t+1}$
        \For{$i = 1, \ldots, n$ \textbf{on Workers in parallel}}
        \State $w^{t+1} = w^{t} + \Delta^{t+1}$ 
        \EndFor
        \EndFor
        \State \textbf{Output:} $x^T$
    \end{algorithmic}
\end{algorithm}
\end{minipage}
\end{figure}

As outlined in Section \ref{sec:ef21-p} of the main text, our primary contribution to the exploration of \algname{EF21-P} is algorithmic. In Algorithm \ref{alg:EF21-P-D}, we extend these results to the distributed setting, allowing for parallel computation of subgradients $\partial f(w^t)$. However, in both Algorithm \ref{alg:EF21-P-S} and \ref{alg:EF21-P-D}, the gradient-like step \eqref{eq:appendix:ef21-nonsmooth} and state update step 
\begin{eqnarray}\label{eq:appendix:ef21-nonsmooth_wtpo}
w^{t+1} = w^t + \Delta^{t+1} 
\end{eqnarray}
remain fundamentally the same.

Given that both single-node and distributed regimes result in the same update steps \eqref{eq:appendix:ef21-nonsmooth} and \eqref{eq:appendix:ef21-nonsmooth_wtpo}, the original proof by \citet{NonsmoothEF21-P} for Algorithm \ref{alg:EF21-P-S} remains applicable to our Algorithm \ref{alg:EF21-P-D}. Nevertheless, for completeness, we provide proofs for all necessary lemmas and theorems, following the approach in \citep{NonsmoothEF21-P}.

Our proof technique proceeds as follows: we first establish two key bounds in Lemma \ref{le:EF21-P-D:main_lemma}. We then combine these bounds to obtain a descent lemma (Lemma \ref{le:EF21-P-D:combination}). Finally, we leverage this descent lemma to establish convergence results (Theorem \ref{thm:EF21-P-D} and Corollary \ref{cor:EF21-P-D}) for different stepsize schedules.

\begin{lemma}[Key bounds]\label{le:EF21-P-D:main_lemma}
Let Assumptions \ref{as:existence_of_minimizer} and \ref{as:fi_convexity} hold. Define $W^t \eqdef \cb{w_1^t, \dots, w_n^t}$.
Then, for a single iteration of \algname{EF21-P} (Algorithm \ref{alg:EF21-P}) with $\gamma_t > 0$, we have the following bounds:

\textbf{1.}
\begin{eqnarray}\label{eq:le:EF21-P-D:main_lemma_term_1}
\Exp{\sqnorm{\xtpo - x^*} \mid x^t, W^t} &\le& \sqnorm{\xt - x^*} - 2\gamma_t \left( f(w^t) - f(x^*) \right) + \frac{1}{\lambda} \sqnorm{w^t - \xt} \nonumber \\
&& + (1 + \lambda) \gamma_t^2 \twonorm{\pf{w^t}}^2,
\end{eqnarray}
where $\lambda>0$;

\textbf{2.}
\begin{eqnarray}\label{eq:le:EF21-P-D:main_lemma_term_2}
 \Exp{\sqnorm{w^{t+1} - x^{t+1}} \mid x^t, W^t} \le  (1 - \theta)  \sqnorm{w^{t} - x^t} + \gamma_t^2 \beta  \twonorm{\pf{w^t}}^2,
\end{eqnarray}
where $\theta \eqdef 1 - \sqrt{1-\alpha}$ and $\beta \eqdef \fr{1 - \alpha}{1 - \sqrt{1-\alpha}}$.
\end{lemma}
\begin{proof}
We prove each bound separately.

\textbf{1.} To establish the first bound, we begin by applying the definition of subgradient:
\begin{eqnarray}
f(x^*) &\geq& f(w^t) + \lin{\pf{w^t}, x^* - w^t},
\end{eqnarray}
which implies:
\begin{eqnarray}
\lin{\pf{w^t}, w^t - x^*} &\geq& f(w^t) - f(x^*).
\end{eqnarray}

Next, we apply \eqref{eq:langles_gen_young} with $s\eqdef \lambda \gamma_t$:
\begin{eqnarray}\label{eq:ef21_lema_1_eq_1}
2\gamma_t \lin{\pf{w^t}, w^t - \xt} &\leq& \lambda \gamma_t^2 \twonorm{\pf{w^t}}^2 + \frac{1}{\lambda} \sqnorm{w^t - \xt}.
\end{eqnarray}
where $\lambda>0$ is a constant to be specified later.

Using the linearity of inner product, we derive:
\begin{eqnarray}\label{eq:ef21_lema_1_eq_2}
-2\gamma_t \lin{\pf{w^t}, x^* - \xt} &=& -2\gamma_t \lin{\pf{w^t}, w^t - x^*} + 2\gamma_t \lin{\pf{w^t}, w^t - \xt} \nonumber \\
&\letext{\eqref{eq:ef21_lema_1_eq_1}}& -2\gamma_t \left( f(w^t) - f(x^*) \right) + \lambda \gamma_t^2 \twonorm{\pf{w^t}}^2 + \frac{1}{\lambda} \sqnorm{w^t - \xt}.
\end{eqnarray}

Finally, we establish the first bound \eqref{eq:le:EF21-P-D:main_lemma_term_1}:

\begin{eqnarray}
\Exp{\sqnorm{\xtpo - x^*} \mid x^t, W^t} &=& \Exp{\sqnorm{\xt - \gamma_t \pf{w^t} - x^*} \mid x^t, W^t} \nonumber \\
&=& \sqnorm{\xt - x^*} - 2\gamma_t \lin{\pf{w^t}, \xt - x^*} + \gamma_t^2 \twonorm{\pf{w^t}}^2 \nonumber \\
&\letext{\eqref{eq:ef21_lema_1_eq_2}}& \sqnorm{\xt - x^*} - 2\gamma_t \left( f(w^t) - f(x^*) \right) + \frac{1}{\lambda} \sqnorm{w^t - \xt} \nonumber \\
&& + (1 + \lambda) \gamma_t^2 \twonorm{\pf{w^t}}^2.
\end{eqnarray}

\textbf{2.} For the second bound, we proceed as follows:

\begin{eqnarray}
 \Exp{\sqnorm{w^{t+1} - x^{t+1}} \mid x^t, W^t} &=&  \Exp{\sqnorm{w^{t} - \cC(x^{t+1} - w^{t}) - x^{t+1}} \mid x^t, W^t} \nonumber \\
&\leq& (1 - \alpha)  \sqnorm{w^{t} - x^{t+1}} \nonumber \\
&=& (1 - \alpha)   \sqnorm{w^{t} - x^t + \gamma_t \pf{w^t}} \nonumber \\
&\leq& (1 - \alpha)(1 + s)  \sqnorm{w^{t} - x^t} + \gamma_t^2 (1 - \alpha)(1 + s^{-1})  \twonorm{\pf{w^t}}^2 \nonumber \\
&\leq& (1 - \theta(s))  \sqnorm{w^{t} - x^t} + \gamma_t^2 \beta(s)  \twonorm{\pf{w^t}}^2, \nonumber \\
\end{eqnarray}
where $\theta(s) \eqdef 1 - (1-\alpha)(1+s)$ and $\beta(s) \eqdef (1-\alpha)(1+s^{-1})$.

Following \citep{richtarik2021ef21}, the optimal $s$, minimizing $\frac{(1-\alpha)(1+1 / s)}{1-(1-\alpha)(1+s)}$, is $s_* = \fr{1}{\sqrt{1-\alpha}} - 1$, resulting in $\theta \eqdef 1-(1-\alpha)(1+s_*) = 1 - \sqrt{1-\alpha}$ and $\beta\eqdef (1-\alpha)(1+1 / s_*) = \fr{1 - \alpha}{1 - \sqrt{1-\alpha}}$.

Therefore, we can establish the second bound \eqref{eq:le:EF21-P-D:main_lemma_term_2}:
\begin{eqnarray}
 \Exp{\sqnorm{w^{t+1} - x^{t+1}} \mid x^t, W^t} \le  (1 - \theta)  \sqnorm{w^{t} - x^t} + \gamma_t^2 \beta  \twonorm{\pf{w^t}}^2.
\end{eqnarray}
\end{proof}

With these two key bounds established in Lemma \ref{le:EF21-P-D:main_lemma}, we can now proceed to the descent lemma. This lemma describes the one-step behavior of Algorithm \ref{alg:EF21-P} for any $\gamma_t>0$ and will be crucial in establishing our convergence rates.

\begin{lemma}[Descent lemma]\label{le:EF21-P-D:combination}
Let the conditions of Lemma \ref{le:EF21-P-D:main_lemma} hold. Define the Lyapunov function
\begin{eqnarray}\label{eq:EF21-P:lyapunov_def}
V^t_{\lambda} \eqdef \sqnorm{\xt - x^*} + \frac{1}{\lambda \theta}  \sqnorm{w^{t} - x^t},
\end{eqnarray}
where $\lambda>0$ and $\theta \eqdef 1 - \sqrt{1-\alpha}$.
 Then 
\begin{eqnarray}
\Exp{V^{t+1}_{\lambda} \mid x^t, W^t} \le V^{t}_{\lambda} - 2\gamma_t \left(f(w^t) - f(x^*)\right) + \rb{1 + \lambda + \fr{\beta} {\lambda \theta}} \gamma_t^2 \twonorm{\pf{w^t}}^2,
\end{eqnarray}
where $\beta \eqdef \fr{1 - \alpha}{1 - \sqrt{1-\alpha}}$.
\end{lemma}

\begin{proof}
Recall that Lemma \ref{le:EF21-P-D:main_lemma} provides us with two key bounds:

\begin{eqnarray}\label{thm:EF21-P-D:main:term_one}
\Exp{\sqnorm{\xtpo - x^*} \mid x^t, W^t} &\le& \sqnorm{\xt - x^*} - 2\gamma_t \left( f(w^t) - f(x^*) \right) + \frac{1}{\lambda} \sqnorm{w^t - \xt} \nonumber \\
&&+ (1 + \lambda) \gamma_t^2 \twonorm{\pf{w^t}}^2,
\end{eqnarray}
and
\begin{eqnarray}\label{thm:EF21-P-D:main:term_two}
\Exp{\sqnorm{w^{t+1} - x^{t+1}} \mid x^t, W^t} \le  (1 - \theta)  \sqnorm{w^{t} - x^t} + \gamma_t^2 \beta  \twonorm{\pf{w^t}}^2.
\end{eqnarray}

To obtain our descent lemma, we combine \eqref{thm:EF21-P-D:main:term_one} with $\fr{1}{\lambda \theta}$ times \eqref{thm:EF21-P-D:main:term_two}:

\begin{eqnarray}
&&\Exp{V^{t+1}_{\lambda} \mid \xt, W^t}\nonumber \\
&& \eqtext{\eqref{eq:EF21-P:lyapunov_def}}\Exp{\sqnorm{\xtpo - x^*} \mid x^t, W^t} + \frac{1}{\lambda \theta}  \Exp{\sqnorm{w^{t+1} - x^{t+1}} \mid x^t, W^t} \nonumber \\
&&\leq \sqnorm{\xt - x^*} - 2\gamma_t \left( f(w^t) - f(x^*) \right) + \frac{1}{\lambda} \sqnorm{w^{t} - x^t} + (1 + \lambda) \gamma_t^2 \twonorm{\pf{w^t}}^2\nonumber \\
&& \quad+ \frac{1}{\lambda \theta} \rb{(1 - \theta)  \sqnorm{w^{t} - x^t} + \gamma_t^2 \beta  \twonorm{\pf{w^t}}^2} \nonumber \\
&&= \sqnorm{\xt - x^*} + \frac{1}{\lambda \theta}  \sqnorm{w^{t} - x^t} - 2\gamma_t \left(f(w^t) - f(x^*)\right) \nonumber \\
&& \quad + (1 + \lambda) \gamma_t^2 \twonorm{\pf{w^t}}^2 + \fr{\gamma_t^2 \beta} {\lambda \theta}  \twonorm{\pf{w^t}}^2 \nonumber \\
&&= V^t_{\lambda} - 2\gamma_t \left(f(w^t) - f(x^*)\right)  + (1 + \lambda) \gamma_t^2 \twonorm{\pf{w^t}}^2 + \fr{\gamma_t^2 \beta} {\lambda \theta}  \twonorm{\pf{w^t}}^2. \nonumber \\
\end{eqnarray}
\end{proof}

\subsection{Proof of Theorem \ref{thm:main:EF21-P-D}}

Having established the descent lemma, we now proceed to the theorem, which characterizes the convergence behavior of \algname{EF21-P} under various stepsize schedules.

Before we state and prove the theorem, it is important to make a notational remark to avoid confusion.

\begin{remark}
In Lemmas \ref{le:EF21-P-D:main_lemma} and \ref{le:EF21-P-D:combination}, we used an auxiliary term $\lambda > 0$ arising from the application of Young's inequality. This term also appeared in the definition of the Lyapunov function $V_{\lambda}^{t}$. In the following theorem, we will show how to choose this $\lambda$ optimally and denote it as $\lambda_*$. Consequently, we define a Lyapunov function $V^{t}$ such that $V^{t} \eqdef V^{t}_{\lambda_*}$. For simplicity of notation, we will use $V^{t}$ instead of $V^{t}_{\lambda_*}$ in the theorem statement and proof.
\end{remark}

\begin{theorem}[Theorem \ref{thm:main:EF21-P-D}]\label{thm:EF21-P-D}
Let Assumptions \ref{as:existence_of_minimizer}, \ref{as:fi_convexity} and \ref{as:fi_lipschitzness} hold. Define a Lyapunov function
$V^t \eqdef \sqnorm{\xt - x^*} + \frac{1}{\lambda_* \theta} \sqnorm{w^t - \xt}$,
where $\lambda_* \eqdef \fr{\sqrt{1-\alpha}} {1 - \sqrt{1-\alpha}}$ and $\theta \eqdef 1 - \sqrt{1-\alpha}$.
Define also a constant 
$B_* \eqdef 1 + 2\fr{\sqrt{1-\alpha}} {1 - \sqrt{1-\alpha}}$.
Let $\cb{w^t}$ be the sequence produced by \algname{EF21-P} (Algorithm \ref{alg:EF21-P}). Define $\avg{w}^T \eqdef \frac{1}{T} \sum_{t=0}^{T-1} w^t$ and $\what{w}^T \eqdef \frac{1}{\sum_{t=0}^{T-1} \gamma_t}\sum_{t=0}^{T-1} \gamma_t w^t$.

     \textbf{1. (Constant stepsize).} If $\gamma_t \eqdef \gamma >0$, then
\begin{eqnarray}\label{eq:EF21-P-D:constant_stepsize_rate}
\Exp{f(\avg{w}^T) - f(x^*)} \le \frac{V^0}{2\gamma T} + \frac{B_* L_0^2 \gamma}{2}.
\end{eqnarray}
If, moreover, optimal $\gamma$ is chosen i.e. 
\begin{eqnarray}
\gamma \eqdef \fr{1}{\sqrt{T}} \sqrt{\fr{V^0}{B_* L_0^2}},
\end{eqnarray}
 then 
\begin{eqnarray}\label{eq:EF21-P-D:constant_stepsize_rate_optimal}
\Exp{f(\avg{w}^T) - f(x^*)} \le \frac{\sqrt{B_* L_0^2 V^0}}{\sqrt{T}}.
\end{eqnarray}

     \textbf{2. (Polyak stepsize).}
If $\gamma_t$ is chosen as
\begin{eqnarray}\label{eq:appndx:gamma_t_polyak}
\gamma_t &\eqdef& \fr{f(w^t) - f(x^*)}{B_* \sqnorm{\pf{w^t}}},
\end{eqnarray}
then
\begin{eqnarray}\label{eq:appndx:EF21-P-D:polyak_stepsize_rate}
\Exp{f(\avg{w}^T) - f(x^*)} \le \frac{\sqrt{B_* L_0^2 V^0}}{\sqrt{T}}.
\end{eqnarray}

     \textbf{3. (Decreasing stepsize).} If $\gamma_t$ is chosen as 
     \begin{eqnarray}
     \gamma_t \eqdef \fr{\gamma_0}{\sqrt{t+1}},
     \end{eqnarray}
      then
\begin{eqnarray}
\Exp{f(\what{w}^T) - f(x^*)} \leq \frac{V^0 + 2\gamma_0^2 B_* L_0^2 \log(T + 1)}{\gamma_0\sqrt{T}}.
\end{eqnarray}
If, moreover, optimal $\gamma_0$ is chosen i.e. 
\begin{eqnarray}
\gamma_0 \eqdef \sqrt{\fr{V_0}{2B_* L_0^2 \log(T+1)}},
\end{eqnarray}
then 
\begin{eqnarray}\label{eq:appndx:EF21-P-D:decr_stepsize_rate}
\Exp{f(\what{w}^T) - f(x^*)} \leq 2\sqrt{2B_* L_0^2 V_0}\sqrt{\frac{\log(T+1)}{T}}.
\end{eqnarray}

\end{theorem}

\begin{proof}
We will prove each part of the theorem separately, starting with some general bounds that will be useful throughout the proof.

From Assumption \ref{as:fi_lipschitzness}, we can infer that $f$ is $L_0$-Lipschitz with $L_0 \le \suminn L_{0,i}$ and
\begin{eqnarray}\label{eq:L_0_bound}
\twonorm{\partial f(x)} &\le L_{0} \quad \forall x \in \mathbb{R}^d.
\end{eqnarray}

Now, we proceed to prove each part of the theorem.

\textbf{1. (Constant stepsize).}
Using \eqref{eq:L_0_bound}, Lemma \ref{le:EF21-P-D:combination}, the tower property of expectation \eqref{eq:tower_property}, and choosing constant stepsize $\gamma_t \eqdef \gamma > 0$, we obtain

\begin{eqnarray}\label{eq:EF21-P-D:descent_eq1}
\Exp{V^{t+1}} \le \Exp{V^{t}} - 2\gamma \Exp{f(w^t) - f(x^*)} + \rb{1 + \lambda + \fr{\beta} {\lambda \theta}} \gamma^2 \Exp{\twonorm{\pf{w^t}}^2},
\end{eqnarray}
where $\lambda>0$, $\theta \eqdef 1 - \sqrt{1-\alpha}$ and $\beta \eqdef \fr{1 - \alpha}{1 - \sqrt{1-\alpha}}$.

From the inequality \eqref{eq:EF21-P-D:descent_eq1}, we have
\begin{eqnarray}\label{eq:EF21-P-D:descent_eq2}
\Exp{V^{t+1}} \leq \Exp{V^t} - 2\gamma \Exp{f(w^t) - f(x^*)} + B_{\lambda} {L}_0^2 \gamma^2,
\end{eqnarray}
where $B_{\lambda} \eqdef 1+\lambda + \fr{\beta}{\lambda \theta}$.

Since $f$ is convex, by Jensen's inequality \eqref{eq:jensen_general}, we have 
\begin{eqnarray}\label{eq:EF21-P:proof:constant_stepsize_bound_gen}
\Exp{f(\what{w}^T) - f(x^*)} &\le& \Exp{\frac{1}{T} \sum_{t=0}^{T-1} f({w}^t) - f(x^*)} \nonumber \\
&\leq& \frac{1}{T} \sum_{t=0}^{T-1} \Exp{f(w^t) - f(x^*)} \nonumber \\
&\letext{\eqref{eq:EF21-P-D:descent_eq2}}& \frac{\Exp{V^0} - \Exp{V^T}}{2\gamma T} + \frac{B_{\lambda}L_0^2 \gamma}{2} \nonumber \\
&\letext{$V^T \geq 0$}& \frac{V^0}{2\gamma T} + \frac{B_{\lambda}L_0^2 \gamma}{2}.
\end{eqnarray}

To optimize this bound, we need to find the optimal $\lambda$. Note that $\phi(\lambda) \eqdef 1+\lambda + \fr{\beta}{\lambda \theta}$ is a convex function on $(0, +\infty)$ for any fixed values $\beta>0$ and $\theta \in (0,1]$.

Therefore, we define the optimal $\lambda$ value (denoted $\lambda_*$) as
\begin{eqnarray}\label{eq:const:EF21_P:lambda_star_choice}
\lambda_* \eqdef \argmin_{\lambda > 0} \left( 1+\lambda + \fr{\beta}{\lambda \theta} \right) = \sqrt{\fr{\beta}{\theta}} = \fr{\sqrt{1-\alpha}} {1 - \sqrt{1-\alpha}}.
\end{eqnarray}

Next, we define the optimal $B_{\lambda}$ value (denoted $B_*$) as
\begin{eqnarray}\label{eq:EF21_P:constant_optimal_B}
B_* \eqdef B_{\lambda_*}= 1 + 2\sqrt{\fr{\beta}{\theta}} = 1+2\fr{\sqrt{1-\alpha}} {1 - \sqrt{1-\alpha}}.
\end{eqnarray}

Plugging \eqref{eq:EF21_P:constant_optimal_B} into \eqref{eq:EF21-P:proof:constant_stepsize_bound_gen}, we get 

\begin{eqnarray}\label{eq:proof:EF21-P:constant_stepsize_bound_opt_lambda} 
\Exp{\suminn f_i(\avg{w}^T) - f(x^*)} \letext{\eqref{eq:EF21_P:constant_optimal_B},  \eqref{eq:EF21-P:proof:constant_stepsize_bound_gen}} \frac{V^0}{2\gamma T} + \frac{B_{*}L_0^2 \gamma}{2}.
\end{eqnarray}

Thus, we have established \eqref{eq:EF21-P-D:constant_stepsize_rate}.

To derive the optimal rate \eqref{eq:EF21-P-D:constant_stepsize_rate_optimal}, we need to find the optimal $\gamma$ stepsize (which we denote $\gamma_*$):

\begin{eqnarray}
\gamma_* \eqdef \argmin_{\gamma} \left( \frac{V^0}{2\gamma T} + \frac{B_*  \gamma}{2} \right) = \frac{1}{ \sqrt{T}} \sqrt{\frac{V^0}{B_* L_0^2}}.
\end{eqnarray}

Therefore, choosing $\gamma \eqdef \gamma_*$, \eqref{eq:proof:EF21-P:constant_stepsize_bound_opt_lambda} reduces to 
\begin{eqnarray}
\Exp{\suminn f_i(\avg{w}^T) - f(x^*)} \le \frac{V^0}{2\gamma_* T} + \frac{B_* \gamma_*}{2} &=& \frac{ \sqrt{V^0 B_* L_0^2}}{\sqrt{T}},
\end{eqnarray}
which gives us \eqref{eq:EF21-P-D:constant_stepsize_rate_optimal}.

\textbf{2. (Polyak stepsize).}

Using Lemma \ref{le:EF21-P-D:combination}, we have
\begin{eqnarray}\label{eq:EF21-P-D:polyak_eq1}
\Exp{V^{t+1}\mid x^t, W^t} \le {V^{t}} - 2\gamma_t {f(w^t) - f(x^*)} + \rb{1 + \lambda + \fr{\beta} {\lambda \theta}} \gamma_t^2 {\twonorm{\pf{w^t}}^2},
\end{eqnarray}
where $\lambda>0$, $\theta \eqdef 1 - \sqrt{1-\alpha}$ and $\beta \eqdef \fr{1 - \alpha}{1 - \sqrt{1-\alpha}}$.

We choose the Polyak stepsize $\gamma_t$ as the one that minimizes the right-hand side of \eqref{eq:EF21-P-D:polyak_eq1}:

\begin{eqnarray}\label{eq:EF21-P-D:gamma_t_expr_1}
\gamma_t &\eqdef& \argmin_{\gamma} \cb{{V^{t}} - 2\gamma \rb{f(w^t) - f(x^*)} + \rb{1 + \lambda + \fr{\beta} {\lambda \theta}} \gamma^2 {\twonorm{\pf{w^t}}^2}}\nonumber \\
&=& \fr{ f(w^t) - f(x^*)}{\rb{1 + \lambda + \fr{\beta} {\lambda \theta}}\sqnorm{\pf{w^t}}}.
\end{eqnarray}

Note that the denominator in \eqref{eq:EF21-P-D:gamma_t_expr_1} is a convex function of $\lambda$.
Therefore, similar to \eqref{eq:const:EF21_P:lambda_star_choice}, we can choose the optimal $\lambda$ as

\begin{eqnarray}
\lambda_* \eqdef \argmin_{\lambda > 0} \left(1 + \lambda + \fr{\beta} {\lambda \theta}  \right) = \sqrt{\fr{\beta}{\theta}},
\end{eqnarray}
and thus
\begin{eqnarray}
B_* = 1 + \lambda_* + \fr{\beta} {\lambda_* \theta} = 1 + 2\sqrt{\fr{\beta}{\theta}} = 1+2\fr{\sqrt{1-\alpha}} {1 - \sqrt{1-\alpha}}.
\end{eqnarray}

Therefore, we derive the final expression for our Polyak stepsize:
\begin{eqnarray}\label{eq:EF21-P-D:gamma_t_expr}
\gamma_t &\eqdef& \fr{ f(w^t) - f(x^*)}{B_* \sqnorm{ \pf{w^t}}}.
\end{eqnarray}

Next, plugging \eqref{eq:EF21-P-D:gamma_t_expr} into \eqref{eq:EF21-P-D:polyak_eq1} and using the tower property of expectation \eqref{eq:tower_property}, we obtain 

\begin{eqnarray}\label{eq:EF21-P_polyak_cons}
\Exp{V^{t+1}} &\letext{\eqref{eq:EF21-P-D:gamma_t_expr_1}, \eqref{eq:EF21-P-D:gamma_t_expr}}& \Exp{V^{t}} - \Exp{\fr{\rb{ f(w^t) - f(x^*)}^2} { \sqnorm{ \pf{w^t}} + 2 \twonorm{ \pf{w^t}}\sqrt{ \twonorm{\pf{w^t}}^2}  \sqrt{\fr{(1-p)\omega}{p}} }} \nonumber \\
&\letext{\eqref{eq:L_0_bound}}& \Exp{V^{t}} - \fr{\Exp{ \rb{ f(w^t) - f(x^*)}^2}}{L_0^2{B}_*},
\end{eqnarray}

Since $f$ is convex, by Jensen's inequality \eqref{eq:jensen_general} and the Cauchy-Bunyakovsky-Schwarz inequality \eqref{eq:cauchy_bunyakovsky_schwarz} with $X\eqdef  f({w}^t) - f(x^*)$ and $Y\eqdef 1$, we have 
\begin{eqnarray}
\Exp{ f_i(\avg{w}^T) - f(x^*)} &\letext{\eqref{eq:jensen_general}}& \Exp{ \frac{1}{T} \sum_{t=0}^{T-1} f(w^t) - f(x^*)} \nonumber \\
&\le& \frac{1}{T}\sum_{t=0}^{T-1} \Exp{ f(w^t) - f(x^*)}\nonumber \\
&\letext{\eqref{eq:cauchy_bunyakovsky_schwarz}}& \frac{1}{T}\sum_{t=0}^{T-1} \sqrt{\Exp{\rb{ f(w^t) - f(x^*)}^2 }}\nonumber \\
&\le& \sqrt{\frac{1}{T}\sum_{t=0}^{T-1}\Exp{\rb{ f(w^t) - f(x^*)}^2 }}\nonumber \\
&\letext{\eqref{eq:EF21-P_polyak_cons}}& \frac{\sqrt{B_*L_0^2}}{\sqrt{T}} \sqrt{{\Exp{V^0} - \Exp{V^T}} } \nonumber \\
&\leq& \frac{\sqrt{V^0B_*L_0^2}}{\sqrt{T}}. \nonumber \\
\end{eqnarray}

Thus, we have established \eqref{eq:appndx:EF21-P-D:polyak_stepsize_rate}.

\textbf{3. (Decreasing stepsize).}

By the same arguments as in the analysis for the constant stepsize case, we can get a bound   

\begin{eqnarray}\label{eq:EF21-P:distributed_nonsmooth_inequality_decreasing}
\Exp{V^{t+1}} \le \Exp{V^t} - 2\gamma_t \Exp{f(w^t) - f(x^*)} + {B}_{*}L_0^2 \gamma_t^2,
\end{eqnarray}
where ${B}_* \eqtext{\eqref{eq:EF21_P:constant_optimal_B}} = 1+2\fr{\sqrt{1-\alpha}} {1 - \sqrt{1-\alpha}}.$

If $\gamma_t \eqdef \frac{\gamma_0}{\sqrt{t+1}}$ with $\gamma_0 > 0$, then we can get the bounds
\begin{eqnarray}\label{eq:EF21-P:decr_stepsize_bounds}
\sum_{t=0}^{T-1} \gamma_t \geq \frac{\gamma_0\sqrt{T}}{2}, \quad \text{and} \quad
\sum_{t=0}^{T-1} \gamma_t^2 \leq 2\gamma_0^2 \log(T + 1).
\end{eqnarray}

Since $f$ is convex, by Jensen's inequality \eqref{eq:jensen_general}, we have 

\begin{eqnarray}\label{eq:EF21-P:sum_decr_st_bound}
\Exp{ f(\what{w}^T) - f(x^*)} &\letext{\eqref{eq:jensen_general}}& \Exp{ \frac{1}{\sum_{t=0}^{T-1} \gamma_t} \sum_{t=0}^{T-1} \gamma_t[f({w}^t) - f(x^*)]} \nonumber \\
&\letext{\eqref{eq:EF21-P:distributed_nonsmooth_inequality_decreasing}}& \frac{(\Exp{V^0} - \Exp{V^T}) + {B}_{*}L_0^2 \sum_{t=0}^{T-1} \gamma_t^2}{2 \sum_{t=0}^{T-1} \gamma_t} \nonumber \\
&\letext{$V^T \geq 0$}& \frac{V^0 + {B}_{*}L_0^2 \sum_{t=0}^{T-1} \gamma_t^2}{2 \sum_{t=0}^{T-1} \gamma_t} \nonumber \\
&\letext{\eqref{eq:EF21-P:decr_stepsize_bounds}}& \frac{V^0 + 2\gamma_0^2 {B}_{*}L_0^2 \log(T + 1)}{\gamma_0\sqrt{T}}.
\end{eqnarray}
The optimal $\gamma_0$ can be chosen by minimizing the right-hand side of \eqref{eq:EF21-P:sum_decr_st_bound}, i.e., 
\begin{eqnarray}
\gamma_* &=& \argmin_{\gamma_0 > 0} \left( \frac{V_0}{\gamma_0 \sqrt{T}} + \frac{2 \gamma_0 {B}_{*}L_0^2 \log(T+1)}{\sqrt{T}} \right) = \sqrt{\frac{V_0}{2 {B}_{*}L_0^2 \log(T+1)}}. \nonumber \\
\end{eqnarray}

Therefore, choosing $\gamma_0 \eqdef \gamma_*$, \eqref{eq:EF21-P:sum_decr_st_bound} reduces to 
\begin{eqnarray}
\Exp{ f(\what{w}^T) - f(x^*)} \le \frac{V_0}{\gamma_* \sqrt{T}} + \frac{2 \gamma_* \log(T+1)}{\sqrt{T}} &=& 2\sqrt{2 V_0} \sqrt{{B}_{*}L_0^2} \sqrt{\frac{\log(T+1)}{T}}, \nonumber \\
\end{eqnarray}
and we get \eqref{eq:appndx:EF21-P-D:decr_stepsize_rate}.
\end{proof}

Having established our main theorem, we can now derive a corollary that provides more practical insights into the performance of \algname{EF21-P}.

\subsection{Proof of Corollary \ref{cor:main:EF21-P-D}}

\begin{corollary}[Corollary \ref{cor:main:EF21-P-D}]\label{cor:EF21-P-D}
Let the conditions of Theorem \ref{thm:main:EF21-P-D} be met and $w^0 = x^0$.
If $\gamma_t$ is set according to \eqref{eq:EF21-P-D:constant_stepsize} or \eqref{eq:gamma_t_polyak} (constant or Polyak stepsizes) 
then \algname{EF21-P} (Algorithm \ref{alg:EF21-P-D}) requires
\begin{eqnarray}\label{cor:eq:ef21_p-complexity}
T = \cO\rb{\frac{L_0^2 R_0^2}{\alpha\eps^2}}
\end{eqnarray}
iterations/communication rounds in order to achieve $\Exp{f(\avg{w}^T) - f(x^*)} \le \eps$.
Moreover, under the assumption that the communication cost is proportional to the number of non-zero components of vectors transmitted from the server to workers, we have that the expected total communication cost per worker equals
\begin{eqnarray}
d + \zeta_{\cC}T = \cO\rb{d + \frac{\zeta_{\cC}L_0^2 R_0^2}{\alpha\eps^2}}.
\end{eqnarray}
\end{corollary}

\begin{proof}
From \eqref{eq:EF21-P-D:constant_stepsize_rate_optimal} and \eqref{eq:appndx:EF21-P-D:polyak_stepsize_rate}, we have the convergence rate 
\begin{eqnarray}\label{eq:ef21-p:rate_constant_polyak}
\Exp{f(\avg{w}^T) - f(x^*)} \le \frac{\sqrt{B_* L_0^2 V^0}}{\sqrt{T}},
\end{eqnarray}
where

$V^0 = \sqnorm{x^0 - x^*} + \frac{1}{\lambda_* \theta} \sqnorm{w^0 - x^0}$,
with $\lambda_* \eqdef \fr{\sqrt{1-\alpha}} {1 - \sqrt{1-\alpha}}$ and $\theta \eqdef 1 - \sqrt{1-\alpha}$.

$B_* \eqdef 1 + 2\fr{\sqrt{1-\alpha}} {1 - \sqrt{1-\alpha}}$, resulting in a complexity 
\begin{eqnarray}\label{eq:ef21-p:complexity_constant_polyak}
T = \cO\rb{\frac{{{B}_*L^2_0 V^0}}{{\eps}^2}}
\end{eqnarray}
required to achieve $\Exp{f(\avg{w}^T) - f(x^*)} \le \eps$.
Assuming $w^0 = x^0$, we get
\begin{eqnarray}\label{eq:ef21_p_substitutions}
V^0 = R^2_0 = \sqnorm{x^0 - x^*}.
\end{eqnarray}

Further, note 
\begin{eqnarray}\label{eq:B_star_ef21-P}
B_* &=& 1 + 2\fr{\sqrt{1-\alpha}} {1 - \sqrt{1-\alpha}} \nonumber \\
&=& 1 + 2\fr{\sqrt{1-\alpha}(1 + \sqrt{1-\alpha})}{\alpha} \nonumber \\
&=& 1 + 2\rb{\fr{\sqrt{1-\alpha} + 1 - \alpha}{\alpha}}\nonumber \\
&\le& \fr{4}{\alpha} -1.
\end{eqnarray}

Plugging \eqref{eq:ef21_p_substitutions} and \eqref{eq:B_star_ef21-P} into \eqref{eq:ef21-p:complexity_constant_polyak}, we get \eqref{cor:eq:ef21_p-complexity}.

The expected total communication cost per worker is
\begin{eqnarray}
d + \zeta_{\cC}T = \cO\rb{d + \frac{\zeta_{\cC}L_0^2 R_0^2}{\alpha\eps^2}}.
\end{eqnarray}

\end{proof}

This concludes our analysis of the \algname{EF21-P} algorithm. We have established its convergence rates for different stepsize schedules and derived complexity bounds. 
In the next section, we will proceed to analyze the \algname{MARINA-P} algorithm.

\section{Missing Proofs For \algname{MARINA-P}}

In this section, we present the detailed proofs for the theoretical results for \algname{MARINA-P} algorithm. Our proof technique proceeds as follows: we first establish two key bounds in Lemma \ref{le:MARINA-P-D:main_lemma}. We then combine these bounds to obtain a descent lemma (Lemma \ref{le:MARINA-P-D:combination}). Finally, we leverage this descent lemma to establish convergence results (Theorem \ref{thm:MARINA-P-D} and Corollary \ref{cor:MARINA-P-D}) for different stepsize schedules.

\begin{lemma}[Key bounds]\label{le:MARINA-P-D:main_lemma}
Let Assumptions \ref{as:existence_of_minimizer} and \ref{as:fi_convexity} hold. Define $W^t \eqdef \cb{w_1^t, \dots, w_n^t}$.
Then, for a single iteration of \algname{MARINA-P} (Algorithm \ref{alg:MARINA-P}) with $\gamma_t > 0$, we have the following bounds:

\textbf{1.}

\begin{eqnarray}\label{eq:le:MARINA-P-D:main_lemma_term_1}
\Exp{\sqnorm{\xtpo - x^*} \mid \xt, W^t} &\le&  \sqnorm{\xt - x^*} - 2\gamma_t \left(\suminn f_i(w_i^t) - f(x^*)\right) + \lambda \gamma_t^2 \suminn \twonorm{\pfi{w_i^t}}^2 \nonumber \\ 
&&+ \frac{1}{\lambda} \suminn \sqnorm{w_i^t - x^t} + \gamma_t^2 \twonorm{\suminn\pfi{w_i^t}}^2,
\end{eqnarray}
where $\lambda>0$;

\textbf{2.}

\begin{eqnarray}\label{eq:le:MARINA-P-D:main_lemma_term_2}
\suminn \Exp{\sqnorm{w_i^{t+1} - x^{t+1}} \mid x^t, W^t} \le (1-p) \suminn \sqnorm{w_i^t - x^t} + (1-p)\omega \gamma_t^2 \sqnorm{\suminn \pfi{w_i^t}}.\nonumber\\
\end{eqnarray}
\end{lemma}
\begin{proof}
We prove each bound separately.

\textbf{1.}

To establish the first bound, we begin by applying the definition of subgradient:
\begin{eqnarray}\label{eq:smooth_convex_bound}
f_i(x^*) \geq f_i(w_i^t) + \lin{\pfi{w_i^t}, x^* - w_i^t} \quad \forall i \in [n].
\end{eqnarray}
Summing over all $i \in [n]$, we obtain
\begin{eqnarray}
\suminn f_i(x^*) \geq \suminn f_i(w_i^t) + \suminn \lin{\pfi{w_i^t}, x^* - w_i^t},
\end{eqnarray}
which implies
\begin{eqnarray}\label{eq:convexity_sum}
\suminn \lin{\pfi{w_i^t}, w_i^t - x^*} \geq \suminn f_i(w_i^t) - f(x^*).
\end{eqnarray}

Next, we apply \eqref{eq:langles_gen_young} with $s\eqdef \lambda \gamma_t$:
\begin{eqnarray}\label{eq:young_langles_application}
2\gamma_t \suminn \lin{\pfi{w_i^t}, w_i^t - x^t} \le \lambda \gamma_t^2 \suminn \twonorm{\pfi{w_i^t}}^2 + \frac{1}{\lambda} \suminn \sqnorm{w_i^t - x^t},
\end{eqnarray}
where $\lambda>0$ is a constant to be specified later.

Using the linearity of inner product, we derive
\begin{eqnarray}\label{eq:key_inner_product}
&&-2\gamma_t \lin{\suminn \pfi{w_i^t}, \xt - x^*} \nonumber \\
&&= -2\gamma_t \suminn \lin{\pfi{w_i^t}, \xt - x^*} \nonumber \\
&&= -2\gamma_t \suminn \lin{\pfi{w_i^t}, w_i^t - x^*} + 2\gamma_t \suminn \lin{\pfi{w_i^t}, w_i^t - \xt} \nonumber \\
&&\letext{\eqref{eq:convexity_sum}, \eqref{eq:young_langles_application}} -2\gamma_t \left(\suminn f_i(w_i^t) - f(x^*)\right) + \lambda \gamma_t^2 \suminn \twonorm{\pfi{w_i^t}}^2 + \frac{1}{\lambda} \suminn \sqnorm{w_i^t - \xt}.\nonumber \\ 
\end{eqnarray}

Finally, we establish the first bound \eqref{eq:le:MARINA-P-D:main_lemma_term_1}:
\begin{eqnarray}
&&\Exp{\sqnorm{\xtpo - x^*} \mid \xt, W^t} \nonumber \\
&&\eqtext{\eqref{eq:alg:marina_p:update}} \Exp{\sqnorm{\xt - \gamma_t \suminn \pfi{w_i^t} - x^*} \mid \xt, W^t} \nonumber \\
&&= \sqnorm{\xt - x^*} - 2\gamma_t \lin{\suminn \pfi{w_i^t}, \xt - x^*} + \gamma_t^2 \sqnorm{\suminn \pfi{w_i^t}} \nonumber \\
&&\letext{\eqref{eq:key_inner_product}} \sqnorm{\xt - x^*} - 2\gamma_t \left(\suminn f_i(w_i^t) - f(x^*)\right) + \lambda \gamma_t^2 \suminn \twonorm{\pfi{w_i^t}}^2 + \frac{1}{\lambda} \suminn \sqnorm{w_i^t - x^t} \nonumber \\ 
&& \quad + \gamma_t^2 \twonorm{\suminn\pfi{w_i^t}}^2.\nonumber \\
\end{eqnarray}

\textbf{2.}

For the second bound, we consider the definition of $w_i^{t+1}$ from step \eqref{eq:wtpo_update} of Algorithm \ref{alg:MARINA-P}:
\begin{eqnarray}
w_i^{t+1} = \begin{cases} 
    x^{t+1} & \text{with probability } p, \\
    w_i^t + \cQ_i^t(x^{t+1} - x^t) & \text{with probability } 1-p.\end{cases}
\end{eqnarray}
Applying the variance decomposition \eqref{eq:bvd} and tower property \eqref{eq:tower_property}, we establish the second bound \eqref{eq:le:MARINA-P-D:main_lemma_term_2}:
\begin{eqnarray}
&&\suminn \Exp{\sqnorm{w_i^{t+1} - x^{t+1}} \mid x^t, W^t} \nonumber \\
&& \eqtext{\eqref{eq:tower_property}} (1-p) \suminn \Exp{\sqnorm{w_i^t + \cQ_i(x^{t+1} - x^t) - x^{t+1}} \mid x^t, W^t} \nonumber \\
&&= (1-p) \suminn \Exp{\sqnorm{w_i^t - x^t - \cQ_i(x^{t+1} - x^t) - (x^{t+1} - x^t)} \mid x^t, W^t} \nonumber \\
&&\eqtext{\eqref{eq:bvd}} (1-p) \suminn \sqnorm{w_i^t - x^t} + (1-p) \suminn \Exp{\sqnorm{\cQ_i(x^{t+1} - x^t) - (x^{t+1} - x^t)} \mid x^t, W^t} \nonumber \\
&&\le (1-p) \suminn \sqnorm{w_i^t - x^t} + (1-p)\omega \sqnorm{x^{t+1} - x^t} \nonumber \\
&&\eqtext{\eqref{eq:alg:marina_p:update}} (1-p) \suminn \sqnorm{w_i^t - x^t} + (1-p)\omega \gamma_t^2 \sqnorm{\suminn \pfi{w_i^t}}.
\end{eqnarray}
\end{proof}

With these two key bounds established in Lemma \ref{le:MARINA-P-D:main_lemma}, we can now proceed to the descent lemma. This lemma describes the one-step behavior of Algorithm \ref{alg:MARINA-P} for any $\gamma_t>0$ and will be crucial in establishing our convergence rates.

\begin{lemma}[Descent lemma]\label{le:MARINA-P-D:combination}
Let the conditions of Lemma \ref{le:MARINA-P-D:main_lemma} hold. Define the Lyapunov function
\begin{eqnarray}\label{eq:lyapunov_def}
V^t_{\lambda} \eqdef \sqnorm{\xt - x^*} + \frac{1}{\lambda p} \suminn \sqnorm{w_i^t - \xt},
\end{eqnarray}
where $\lambda>0$ is a constant. Then 
 \begin{eqnarray}
\Exp{V^{t+1}_{\lambda} \mid x^t, W^t} &\le& V^{t}_{\lambda} - 2\gamma_t \left(\suminn f_i(w_i^t) - f(x^*)\right) + \lambda \gamma_t^2 \suminn \twonorm{\pfi{w_i^t}}^2 \nonumber \\ 
&& \quad + \gamma_t^2 \left(1 + \frac{(1-p)\omega}{p\lambda}\right) \sqnorm{\suminn \pfi{w_i^t}}.
\end{eqnarray}
\end{lemma}

\begin{proof}
Recall that Lemma \ref{le:MARINA-P-D:main_lemma} provides us with two key bounds:
\begin{enumerate}
\item
\begin{eqnarray}\label{thm:MARINA-P-D:main:term_one}
\Exp{\sqnorm{\xtpo - x^*} \mid \xt, W^t} &\leq& \sqnorm{\xt - x^*} - 2\gamma_t \left(\suminn f_i(w_i^t) - f(x^*)\right) + \frac{1}{\lambda} \suminn \sqnorm{w_i^t - \xt} \nonumber \\
&& + \lambda \gamma_t^2 \suminn \twonorm{\pfi{w_i^t}}^2 + \gamma_t^2 \sqnorm{\suminn \pfi{w_i^t}},\nonumber \\ 
\end{eqnarray}
where $\lambda>0$;
\item
\begin{eqnarray}\label{thm:MARINA-P-D:main:term_two}
\suminn \Exp{\sqnorm{w_i^{t+1} - x^{t+1}} \mid \xt, W^t} &\leq& (1-p) \suminn \sqnorm{w_i^t - \xt} + (1-p)\omega \gamma_t^2 \sqnorm{\suminn \pfi{w_i^t}}.\nonumber \\ 
\end{eqnarray}
\end{enumerate}

To obtain our descent lemma, we combine \eqref{thm:MARINA-P-D:main:term_one} with $\fr{1}{\lambda p}$ times \eqref{thm:MARINA-P-D:main:term_two}:
\begin{eqnarray}
&&\Exp{V^{t+1}_{\lambda} \mid \xt, W^t}\nonumber \\
&& \eqtext{\eqref{eq:lyapunov_def}} \Exp{\sqnorm{\xtpo - x^*} \mid \xt, W^t} + \frac{1}{\lambda p} \suminn \Exp{\sqnorm{w_i^{t+1} - x^{t+1}} \mid \xt, W^t} \nonumber \\
&&\leq \sqnorm{\xt - x^*} - 2\gamma_t \left(\suminn f_i(w_i^t) - f(x^*)\right) + \frac{1}{\lambda} \suminn \sqnorm{w_i^t - \xt} \nonumber \\
&& \quad+ \lambda \gamma_t^2 \suminn \twonorm{\pfi{w_i^t}}^2 + \gamma_t^2 \sqnorm{\suminn \pfi{w_i^t}} \nonumber \\
&& \quad+ \frac{1}{\lambda p} \left((1-p) \suminn \sqnorm{w_i^t - \xt} + (1-p)\omega \gamma_t^2 \sqnorm{\suminn \pfi{w_i^t}}\right)
\end{eqnarray}
\begin{eqnarray}
&&= \sqnorm{\xt - x^*} + \frac{1}{\lambda p} \suminn \sqnorm{w_i^t - \xt} - 2\gamma_t \left(\suminn f_i(w_i^t) - f(x^*)\right) \nonumber \\
&& \quad+ \lambda \gamma_t^2 \suminn \twonorm{\pfi{w_i^t}}^2 + \gamma_t^2 \left(1 + \frac{(1-p)\omega}{p\lambda}\right) \sqnorm{\suminn \pfi{w_i^t}} \nonumber \\
&&\eqtext{\eqref{eq:lyapunov_def}} V_{\lambda}^t - 2\gamma_t \left(\suminn f_i(w_i^t) - f(x^*)\right) \nonumber \\
&& \quad+ \lambda \gamma_t^2 \suminn \twonorm{\pfi{w_i^t}}^2 + \gamma_t^2 \left(1 + \frac{(1-p)\omega}{p\lambda}\right) \sqnorm{\suminn \pfi{w_i^t}}. \nonumber \\
\end{eqnarray}

This completes the proof of the descent lemma.
\end{proof}

\subsection{Proof of the Theorem \ref{thm:main:MARINA-P-D}}
With the descent lemma established, we can now proceed to the main theoretical result of this paper. Before we state and prove the theorem, it is important to make a notational remark to avoid confusion.

\begin{remark}
In Lemmas \ref{le:MARINA-P-D:main_lemma} and \ref{le:MARINA-P-D:combination}, we used an auxiliary term $\lambda > 0$ arising from the application of Young's inequality. This term also appeared in the definition of the Lyapunov function $V_{\lambda}^{t}$. In the following theorem, we will show how to choose this $\lambda$ optimally and denote it as $\lambda_*$. Consequently, we define a Lyapunov function $V^{t}$ such that $V^{t} \eqdef V^{t}_{\lambda_*}$. For simplicity of notation, we will use $V^{t}$ instead of $V^{t}_{\lambda_*}$ in the theorem statement and proof.
\end{remark}

Now, let us restate and prove the main theorem.
\begin{theorem}[Theorem \ref{thm:main:MARINA-P-D}]\label{thm:MARINA-P-D}
Let Assumptions \ref{as:existence_of_minimizer}, \ref{as:fi_convexity} and \ref{as:fi_lipschitzness} hold. Define a Lyapunov function
$V^t \eqdef \sqnorm{\xt - x^*} + \frac{1}{\lambda_* p} \suminn \sqnorm{w_i^t - \xt}$,
where $\lambda_* \eqdef \fr{\avg{L}_0}{\wt{L}_0}\sqrt{\fr{(1-p)\omega}{p}}$.
Define also a constant 
$\wt{B}_* \eqdef \avg{L}_0^2 + 2{\avg{L}_0}{\wt{L}_0}\sqrt{\fr{(1-p)\omega}{p}}$.
Let $\cb{w_i^t}$ be the sequence produced by \algname{MARINA-P} (Algorithm \ref{alg:MARINA-P}). Define $\avg{w}_i^T \eqdef \frac{1}{T} \sum_{t=0}^{T-1} w_i^t$ and $\what{w}_i^T \eqdef \frac{1}{\sum_{t=0}^{T-1} \gamma_t}\sum_{t=0}^{T-1} \gamma_t w_i^t$ for all $i \in [n]$. 

\textbf{1. (Constant stepsize).} If $\gamma_t \eqdef \gamma >0$, then
\begin{eqnarray}\label{eq:MARINA-P-D:constant_stepsize_rate}
\Exp{\suminn f_i(\avg{w}_i^T) - f(x^*)} \le \frac{V^0}{2\gamma T} + \frac{\wt{B}_* \gamma}{2}.
\end{eqnarray}
If, moreover, the optimal $\gamma$ is chosen, i.e., 
\begin{eqnarray}
\gamma \eqdef \fr{1}{\sqrt{T}} \sqrt{\fr{V^0}{\wt{B}_*}},
\end{eqnarray}
 then 
\begin{eqnarray}\label{eq:MARINA-P-D:constant_stepsize_rate_optimal}
\Exp{\suminn f_i(\avg{w}_i^T) - f(x^*)} \le \frac{\sqrt{\wt{B}_* V^0}}{\sqrt{T}}.
\end{eqnarray}

\textbf{2. (Polyak stepsize).}
If $\gamma_t$ is chosen as
\begin{eqnarray}
\gamma_t &\eqdef& \fr{\suminn f_i(w_i^t) - f(x^*)}{\sqnorm{\suminn \pfi{w_i^t}} + 2 \twonorm{\suminn \pfi{w_i^t}}\sqrt{\suminn \twonorm{\pfi{w_i^t}}^2}  \sqrt{\fr{(1-p)\omega}{p}}},
\end{eqnarray}
then
\begin{eqnarray}\label{eq:MARINA-P-D:polyak_stepsize_rate}
\Exp{\suminn f_i(\avg{w}_i^T) - f(x^*)} \le \frac{\sqrt{\wt{B}_* V^0}}{\sqrt{T}}.
\end{eqnarray}

\textbf{3. (Decreasing stepsize).} If $\gamma_t$ is chosen as 
\begin{eqnarray}
\gamma_t \eqdef \fr{\gamma_0}{\sqrt{t+1}},
\end{eqnarray}
 then
\begin{eqnarray}\label{eq:MARINA-P-D:decreasing_stepsize_rate}
\Exp{\suminn f_i(\avg{w}_i^T) - f(x^*)} \leq \frac{V^0 + 2\gamma_0^2 \wt{B}_* \log(T + 1)}{\gamma_0\sqrt{T}}.
\end{eqnarray}
If, moreover, the optimal $\gamma_0$ is chosen, i.e.,
\begin{eqnarray}
\gamma_0 \eqdef \sqrt{\fr{V_0}{2\wt{B}_* \log(T+1)}},
\end{eqnarray}
 then 
\begin{eqnarray}\label{eq:MARINA-P-D:decr_stepsize_rate}
\Exp{\suminn f_i(\avg{w}_i^T) - f(x^*)} \leq 2\sqrt{2\wt{B}_* V_0}\sqrt{\frac{\log(T+1)}{T}}.
\end{eqnarray}
\end{theorem}
\begin{proof}
We will prove each part of the theorem separately, starting with some general bounds that will be useful throughout the proof.

From Assumption \ref{as:fi_lipschitzness}, we can infer that
\begin{eqnarray}
\twonorm{\partial f_i(x)} &\le L_{0,i} \quad \forall x \in \mathbb{R}^d \;\;\text{and}\;\; \forall i\in [n].
\end{eqnarray}
This implies
\begin{eqnarray}\label{eq:proof_wtL0_bound}
\suminn \twonorm{\pfi{w_i^t}}^2 &\le& \wt{L}_0^2, \quad \forall w_i^t \in \R^d \text{ and } i\in [n],
\end{eqnarray}
where $\wt{L}_0 \eqdef \sqrt{\suminn L_{0,i}^2}$,
and 
\begin{eqnarray}\label{eq:proof_avgL0_bound}
\twonorm{\suminn \pfi{w_i^t}} &\letext{\eqref{eq:jensen_1}}& \suminn\twonorm{\pfi{w_i^t}} \le \avg{L}_0, \quad \forall w_i^t \in \R^d \text{ and } i\in [n],
\end{eqnarray}
where $\avg{L}_0 \eqdef \suminn L_{0,i}$.

Now, we proceed to prove each part of the theorem.

\textbf{1. (Constant stepsize).}

Using \eqref{eq:proof_wtL0_bound}, \eqref{eq:proof_avgL0_bound}, Lemma \ref{le:MARINA-P-D:combination}, the tower property of expectation \eqref{eq:tower_property}, and choosing constant stepsize $\gamma_t \eqdef \gamma > 0$, we obtain

\begin{eqnarray}\label{eq:distributed_nonsmooth_inequality}
&&\Exp{V^{t+1}} \nonumber \\
&&\leq \Exp{V^t} - 2\gamma \Exp{\suminn f_i(w_i^t) - f(x^*)} + \lambda \gamma^2 \suminn \Exp{\twonorm{\pfi{w_i^t}}^2} \nonumber \\
&&\quad + \left(1 + \frac{(1-p)\omega}{p\lambda}\right) \gamma^2 \Exp{\sqnorm{\suminn \pfi{w_i^t}}} \nonumber \\
&&\letext{\eqref{eq:proof_wtL0_bound}, \eqref{eq:proof_avgL0_bound}} \Exp{V^t} - 2\gamma \Exp{\suminn f_i(w_i^t) - f(x^*)} + \wt{B}_{\lambda} \gamma^2,
\end{eqnarray}
where $\wt{B}_{\lambda} \eqdef \lambda\wt{L}_0^2 + \avg{L}_0^2\rb{1 + \fr{(1-p)\omega}{\lambda p}}$.

Since each $f_i$ for all $i \in [n]$ is convex, by Jensen's inequality \eqref{eq:jensen_general}, we have 
\begin{eqnarray}\label{eq:proof:constant_stepsize_bound_gen}
\Exp{\suminn f_i(\avg{w}_i^T) - f(x^*)} &\letext{\eqref{eq:jensen_general}}& \Exp{\suminn \frac{1}{T} \sum_{t=0}^{T-1} f_i({w}_i^t) - f(x^*)} \nonumber \\
&\leq& \frac{1}{T} \sum_{t=0}^{T-1} \Exp{\suminn f(w_i^t) - f(x^*)} \nonumber \\
&\letext{\eqref{eq:distributed_nonsmooth_inequality}}& \frac{\Exp{V^0} - \Exp{V^T}}{2\gamma T} + \frac{\wt{B}_{\lambda} \gamma}{2} \nonumber \\
&\letext{$V^T \geq 0$}& \frac{V^0}{2\gamma T} + \frac{\wt{B}_{\lambda} \gamma}{2}.
\end{eqnarray}

To optimize this bound, we need to find the optimal $\lambda$. Note that $\phi(\lambda) \eqdef  \lambda \wt{L}_0^2 + \avg{L}_0^2\rb{1 + \fr{(1-p)\omega}{\lambda p}}$ is a convex function on $(0, +\infty)$ for any fixed values $\wt{L}_0>0$, $\avg{L}_0>0$, $p \in (0,1]$, $\omega>0$.

Therefore, we define the optimal $\lambda$ value (denoted $\lambda_*$) as
\begin{eqnarray}\label{eq:const:lambda_star_choice}
\lambda_* \eqdef \argmin_{\lambda > 0} \left( \lambda\wt{L}_0^2 + \avg{L}_0^2\rb{1 + \fr{(1-p)\omega}{\lambda p}} \right) = \fr{\avg{L}_0}{\wt{L}_0}\sqrt{\fr{(1-p)\omega}{p}}.
\end{eqnarray}

Next, we define the optimal $\wt{B}$ value (denoted $\wt{B}_*$) as
\begin{eqnarray}\label{eq:constant_optimal_B}
\wt{B}_* \eqdef \wt{B}_{\lambda_*}= \lambda_*\wt{L}_0^2 + \avg{L}_0^2\rb{1 + \fr{(1-p)\omega}{\lambda_* p}} = \avg{L}_0^2 + 2{\avg{L}_0}{\wt{L}_0}\sqrt{\fr{(1-p)\omega}{p}}.
\end{eqnarray}

Plugging \eqref{eq:constant_optimal_B} into \eqref{eq:proof:constant_stepsize_bound_gen}, we get 

\begin{eqnarray}\label{eq:proof:constant_stepsize_bound_opt_lambda}
\Exp{\suminn f_i(\avg{w}_i^T) - f(x^*)} \letext{\eqref{eq:constant_optimal_B},  \eqref{eq:proof:constant_stepsize_bound_gen}} \frac{V^0}{2\gamma T} + \frac{\wt{B}_{*} \gamma}{2}.
\end{eqnarray}

Thus, we have established \eqref{eq:MARINA-P-D:constant_stepsize_rate}.

To derive the optimal rate \eqref{eq:MARINA-P-D:constant_stepsize_rate_optimal}, we need to find the optimal $\gamma$ stepsize (which we denote $\gamma_*$):

\begin{eqnarray}
\gamma_* \eqdef \argmin_{\gamma} \left( \frac{V^0}{2\gamma T} + \frac{\wt{B}_*  \gamma}{2} \right) = \frac{1}{ \sqrt{T}} \sqrt{\frac{V^0}{\wt{B}_*}}.
\end{eqnarray}

Therefore, choosing $\gamma \eqdef \gamma_*$, \eqref{eq:proof:constant_stepsize_bound_opt_lambda} reduces to 
\begin{eqnarray}
\Exp{\suminn f_i(\avg{w}_i^T) - f(x^*)} \le \frac{V^0}{2\gamma_* T} + \frac{\wt{B}_* \gamma_*}{2} &=& \frac{ \sqrt{V^0 \wt{B}_*}}{\sqrt{T}},
\end{eqnarray}
which gives us \eqref{eq:MARINA-P-D:constant_stepsize_rate_optimal}.

\textbf{2. (Polyak stepsize).}
By Lemma \ref{le:MARINA-P-D:combination}, we have
\begin{eqnarray}\label{eq:distributed_theorem1_inequality_polyak}
\Exp{V^{t+1} \mid x^t, W^t} &\le& V^{t} - 2\gamma_t \left(\suminn f_i(w_i^t) - f(x^*)\right) + \lambda \gamma_t^2 \suminn \twonorm{\pfi{w_i^t}}^2 \nonumber \\
&&\quad + \gamma_t^2 \left(1 + \frac{(1-p)\omega}{p\lambda}\right) \sqnorm{\suminn \pfi{w_i^t}}.
\end{eqnarray}

We choose the Polyak stepsize $\gamma_t$ as the one that minimizes the right-hand side of \eqref{eq:distributed_theorem1_inequality_polyak}:
\begin{eqnarray}\label{eq:gamma_t_expr_1}
\gamma_t &\eqdef& \argmin_{\gamma} \left\{ V^{t} - 2\gamma \left(\suminn f_i(w_i^t) - f(x^*)\right) + \lambda \gamma^2 \suminn \twonorm{\pfi{w_i^t}}^2 \right.\nonumber \\
&& \left.+ \rb{1 + \frac{(1-p)\omega}{p\lambda}}\gamma^2 \sqnorm{\suminn \pfi{w_i^t}}\right\}\nonumber \\
&=& \fr{\suminn f_i(
w_i^t) - f(x^*)}{\lambda \suminn \twonorm{\pfi{w_i^t}}^2 + \rb{1 + \frac{(1-p)\omega}{p\lambda}}\sqnorm{\suminn \pfi{w_i^t}}}.
\end{eqnarray}

Note that the denominator in \eqref{eq:gamma_t_expr_1} is a convex function of $\lambda$.
Therefore, similar to \eqref{eq:const:lambda_star_choice}, we can choose the optimal $\lambda$ as

\begin{eqnarray}
\lambda_* &\eqdef& \argmin_{\lambda > 0} \left(\lambda \suminn \twonorm{\pfi{w_i^t}}^2 + \rb{1 + \frac{(1-p)\omega}{p\lambda}}\sqnorm{\suminn \pfi{w_i^t}}  \right) \nonumber \\
&=& \fr{\twonorm{\suminn \pfi{w_i^t}}}{ \sqrt{\suminn \twonorm{\pfi{w_i^t}}^2} }\sqrt{\fr{(1-p)\omega}{p}},
\end{eqnarray}
and thus
\begin{eqnarray}
&&\lambda_* \suminn \twonorm{\pfi{w_i^t}}^2 + \rb{1 + \frac{(1-p)\omega}{p\lambda_*}}\sqnorm{\suminn \pfi{w_i^t}} \nonumber \\
&& = \rb{\fr{\twonorm{\suminn \pfi{w_i^t}}}{ \sqrt{\suminn \twonorm{\pfi{w_i^t}}^2} }\sqrt{\fr{(1-p)\omega}{p}}}\suminn \twonorm{\pfi{w_i^t}}^2 \nonumber \\
&& \quad + \rb{1 + \frac{(1-p)\omega}{p \rb{\fr{\twonorm{\suminn \pfi{w_i^t}}}{ \sqrt{\suminn \twonorm{\pfi{w_i^t}}^2} }\sqrt{\fr{(1-p)\omega}{p}}}}}\sqnorm{\suminn \pfi{w_i^t}}\nonumber \\
 && = \sqnorm{\suminn \pfi{w_i^t}} + 2 \twonorm{\suminn \pfi{w_i^t}}\sqrt{\suminn \twonorm{\pfi{w_i^t}}^2}  \sqrt{\fr{(1-p)\omega}{p}}.
\end{eqnarray}

Therefore, we derive the final expression for our Polyak stepsize:
\begin{eqnarray}\label{eq:gamma_t_expr}
\gamma_t &\eqdef& \fr{\suminn f_i(w_i^t) - f(x^*)}{\sqnorm{\suminn \pfi{w_i^t}} + 2 \twonorm{\suminn \pfi{w_i^t}}\sqrt{\suminn \twonorm{\pfi{w_i^t}}^2}  \sqrt{\fr{(1-p)\omega}{p}}}.
\end{eqnarray}

Next, plugging \eqref{eq:gamma_t_expr} into \eqref{eq:distributed_theorem1_inequality_polyak} and using the tower property of expectation \eqref{eq:tower_property}, we obtain 

\begin{eqnarray}\label{eq:distributed_theorem1_inequality2_polyak}
&&\Exp{V^{t+1}} \nonumber \\
&&\letext{\eqref{eq:distributed_theorem1_inequality_polyak}, \eqref{eq:gamma_t_expr}} \Exp{V^{t}} - \Exp{\fr{\rb{\suminn f_i(w_i^t) - f(x^*)}^2} { \sqnorm{\suminn \pfi{w_i^t}} + 2 \twonorm{\suminn \pfi{w_i^t}}\sqrt{\suminn \twonorm{\pfi{w_i^t}}^2}  \sqrt{\fr{(1-p)\omega}{p}} }} \nonumber \\
&&\letext{\eqref{eq:proof_wtL0_bound}, \eqref{eq:proof_avgL0_bound}, \eqref{eq:constant_optimal_B}} \Exp{V^{t}} - \fr{\Exp{ \rb{\suminn f_i(w_i^t) - f(x^*)}^2}}{\wt{B}_*},
\end{eqnarray}
where $\wt{B}_* \eqtext{\eqref{eq:constant_optimal_B}} \avg{L}_0^2 + 2{\avg{L}_0}{\wt{L}_0}\sqrt{\fr{(1-p)\omega}{p}}.$

Since each $f_i$ for all $i \in [n]$ is convex, by Jensen's inequality \eqref{eq:jensen_general} and the Cauchy-Bunyakovsky-Schwarz inequality \eqref{eq:cauchy_bunyakovsky_schwarz} with $X\eqdef \suminn f_i({w}_i^t) - f(x^*)$ and $Y\eqdef 1$, we have 
\begin{eqnarray}
\Exp{\suminn f_i(\avg{w}_i^T) - f(x^*)} &\letext{\eqref{eq:jensen_general}}& \Exp{\suminn \frac{1}{T} \sum_{t=0}^{T-1} f_i({w}_i^t) - f(x^*)} \nonumber \\
&\le& \frac{1}{T}\sum_{t=0}^{T-1} \Exp{\suminn f_i({w}_i^t) - f(x^*)}\nonumber \\
&\letext{\eqref{eq:cauchy_bunyakovsky_schwarz}}& \frac{1}{T}\sum_{t=0}^{T-1} \sqrt{\Exp{\rb{\suminn f_i({w}_i^t) - f(x^*)}^2 }}\nonumber \\
&\le& \sqrt{\frac{1}{T}\sum_{t=0}^{T-1}\Exp{\rb{\suminn f_i({w}_i^t) - f(x^*)}^2 }}\nonumber \\
&\letext{\eqref{eq:distributed_theorem1_inequality2_polyak}}& \frac{\sqrt{\wt{B}}}{\sqrt{T}} \sqrt{{\Exp{V^0} - \Exp{V^T}} } \nonumber \\
&\leq& \frac{\sqrt{\wt{B}}\sqrt{V^0}}{\sqrt{T}}. \nonumber \\
\end{eqnarray}

Thus, we have established \eqref{eq:MARINA-P-D:polyak_stepsize_rate}.

\textbf{3. (Decreasing stepsize).}
By the same arguments as in the analysis for the constant stepsize case, we can get a bound   

\begin{eqnarray}\label{eq:distributed_nonsmooth_inequality_decreasing}
\Exp{V^{t+1}} \letext{\eqref{eq:constant_optimal_B}} \Exp{V^t} - 2\gamma_t \Exp{\suminn f_i(w_i^t) - f(x^*)} + \wt{B}_{*} \gamma_t^2,
\end{eqnarray}
where $\wt{B}_* \eqtext{\eqref{eq:constant_optimal_B}} \avg{L}_0^2 + 2{\avg{L}_0}{\wt{L}_0}\sqrt{\fr{(1-p)\omega}{p}}.$

If $\gamma_t \eqdef \frac{\gamma_0}{\sqrt{t+1}}$ with $\gamma_0 > 0$, then we can get the bounds
\begin{eqnarray}\label{eq:decr_stepsize_bounds}
\sum_{t=0}^{T-1} \gamma_t \geq \frac{\gamma_0\sqrt{T}}{2}, \quad \text{and} \quad
\sum_{t=0}^{T-1} \gamma_t^2 \leq 2\gamma_0^2 \log(T + 1).
\end{eqnarray}

Since each $f_i$ for all $i \in [n]$ is convex, by Jensen's inequality \eqref{eq:jensen_general}, we have 

\begin{eqnarray}\label{eq:sum_decr_st_bound}
\Exp{\suminn f_i(\what{w}^T_i) - f(x^*)} &\letext{\eqref{eq:jensen_general}}& \Exp{\suminn \frac{1}{\sum_{t=0}^{T-1} \gamma_t} \sum_{t=0}^{T-1} \gamma_t[f_i({w}_i^t) - f(x^*)]} \nonumber \\
&\letext{\eqref{eq:distributed_nonsmooth_inequality_decreasing}}& \frac{(\Exp{V^0} - \Exp{V^T}) + \wt{B}_{*} \sum_{t=0}^{T-1} \gamma_t^2}{2 \sum_{t=0}^{T-1} \gamma_t} \nonumber \\
&\letext{$V^T \geq 0$}& \frac{V^0 + \wt{B}_{*} \sum_{t=0}^{T-1} \gamma_t^2}{2 \sum_{t=0}^{T-1} \gamma_t} \nonumber \\
&\letext{\eqref{eq:decr_stepsize_bounds}}& \frac{V^0 + 2\gamma_0^2 \wt{B}_{*} \log(T + 1)}{\gamma_0\sqrt{T}}.
\end{eqnarray}
The optimal $\gamma_0$ can be chosen by minimizing the right-hand side of \eqref{eq:sum_decr_st_bound}, i.e., 
\begin{eqnarray}
\gamma_* &=& \argmin_{\gamma_0 > 0} \left( \frac{V_0}{\gamma_0 \sqrt{T}} + \frac{2 \gamma_0 \wt{B}_* \log(T+1)}{\sqrt{T}} \right) = \sqrt{\frac{V_0}{2 \wt{B}_* \log(T+1)}}, \nonumber \\
\end{eqnarray}

Therefore, choosing $\gamma_0 \eqdef \gamma_*$, \eqref{eq:sum_decr_st_bound} reduces to 
\begin{eqnarray}
\Exp{\suminn f_i(\what{w}^T_i) - f(x^*)} \le \frac{V_0}{\gamma_* \sqrt{T}} + \frac{2 \gamma_* \log(T+1)}{\sqrt{T}} &=& 2\sqrt{2 V_0} \sqrt{\wt{B}_*} \sqrt{\frac{\log(T+1)}{T}}, \nonumber \\
\end{eqnarray}
and we get \eqref{eq:MARINA-P-D:decr_stepsize_rate}.
\end{proof}

Having established our main theorem, we can now derive a corollary that provides more practical insights into the performance of \algname{MARINA-P}.

\subsection{Proof of the Corollary \ref{cor:main:MARINA-P-D}}

\begin{corollary}[Corollary \ref{cor:main:MARINA-P-D}]\label{cor:MARINA-P-D}
Let the conditions of Theorem \ref{thm:main:MARINA-P-D} be met, $p = \fr{\zeta_{\cQ}}{d}$ and $w_i^0 = x^0$ for all $i \in [n]$.
If $\gamma_t$ is set according to \eqref{eq:MARINA-P:constant_stepsize} or \eqref{eq:marina_polyak_type} (constant or Polyak stepsizes) 
then \algname{MARINA-P} (Algorithm \ref{alg:MARINA-P}) requires
\begin{eqnarray}\label{cor:eq:marina_p-complexity}
T = \cO\rb{\frac{R^2_0}{\eps^2}\rb{\avg{L}_0^2 + {\avg{L}_0}{\wt{L}_0}\sqrt{\omega\rb{\fr{d}{\zeta_{\cQ}}-1}}}}
\end{eqnarray}
iterations/communication rounds in order to achieve $\Exp{\suminn f_i(\avg{w}_i^T) - f(x^*)} \le \eps$.
Moreover, under the assumption that the communication cost is proportional to the number of non-zero components of vectors transmitted from the server to workers, we have that the expected total communication cost per worker equals
\begin{eqnarray}
\cO\rb{d + \frac{\wt{L}_0^2 R^2_0 \zeta_{\cQ}}{\eps^2}\rb{1 + \sqrt{\omega\rb{\fr{d}{\zeta_{\cQ}}-1}}}}.
\end{eqnarray}
\end{corollary}
\begin{proof}
From \eqref{eq:MARINA-P-D:constant_stepsize_rate_optimal} and \eqref{eq:MARINA-P-D:polyak_stepsize_rate}, we have the convergence rate 
\begin{eqnarray}\label{eq:rate_constant_polyak}
\Exp{\suminn f_i(\avg{w}_i^T) - f(x^*)} \le \frac{\sqrt{\wt{B}_* V^0}}{\sqrt{T}},
\end{eqnarray}
where
$V^0 = \sqnorm{x^0 - x^*} + \frac{1}{\lambda_* p} \suminn \sqnorm{w_i^0 - x^0}$ and  
$\wt{B}_* = \avg{L}_0^2 + 2{\avg{L}_0}{\wt{L}_0}\sqrt{\fr{(1-p)\omega}{p}}$, resulting in a complexity 
\begin{eqnarray}\label{eq:complexity_constant_polyak}
T = \cO\rb{\frac{{\wt{B}_* V^0}}{{\eps}^2}}
\end{eqnarray}
required to achieve $\Exp{\suminn f_i(\avg{w}_i^T) - f(x^*)} \le \eps$.

Assuming $w_i^0 = x^0$ for all $i \in [n]$ and $p = \fr{\zeta_{\cQ}}{d}$, we get
\begin{eqnarray}\label{eq:marina_p_substitutions}
V^0 = R^2_0 = \sqnorm{x^0 - x^*} \quad \text{ and } \quad \wt{B}_* = \avg{L}_0^2 + {\avg{L}_0}{\wt{L}_0}\sqrt{\omega\rb{\fr{d}{\zeta_{\cQ}}-1}}.
\end{eqnarray}

Plugging \eqref{eq:marina_p_substitutions} into \eqref{eq:complexity_constant_polyak}, we get \eqref{cor:eq:marina_p-complexity}.

The expected total communication cost per worker is
\begin{eqnarray}
d + (dp + \zeta_{\cQ}(1-p)) T &\le& d + \frac{\wt{L}_0^2 R^2_0}{\eps^2}(dp + \zeta_{\cQ}(1-p))\rb{1 + 2\sqrt{\fr{(1-p)\omega}{p}}} \nonumber \\
&=& d + \frac{\wt{L}_0^2 R^2_0}{\eps^2}\rb{dp + {\zeta_{\cQ}}(1-p)}\rb{1 + 2\sqrt{\omega\rb{\fr{d}{\zeta_{\cQ}}-1}}} \nonumber \\
&\le& d + \frac{2\wt{L}_0^2 R^2_0}{\eps^2}\zeta_{\cQ}\rb{1 + 2\sqrt{\omega\rb{\fr{d}{\zeta_{\cQ}}-1}}}\nonumber \\
&=& \cO\rb{d + \frac{\wt{L}_0^2 R^2_0 \zeta_{\cQ}}{\eps^2}\rb{1 + \sqrt{\omega\rb{\fr{d}{\zeta_{\cQ}}-1}}}},
\end{eqnarray}
where we used the bound
$p + {\zeta_{\cQ}}(1-p) \le 2 \zeta_{\cQ}$.
\end{proof}




\end{document}